\documentclass[twoside,11pt]{article}

\usepackage[preprint]{jmlr2e}

\usepackage{amsmath}
\usepackage{bm}
\usepackage{subcaption}

\DeclareMathOperator*{\Minimize}{minimize}
\DeclareMathOperator{\Diag}{diag}

\DeclareMathOperator{\Trace}{trace}
\DeclareMathOperator{\NormalDist}{\mathcal{N}}

\newcommand{\RR}{\mathbb{R}}
\newcommand{\CP}{\mathcal{P}}

\newcommand{\Card}[1]{\lvert#1\rvert}
\renewcommand{\Vec}[1]{\bm{#1}}
\newcommand{\RowVec}[1]{\bm{\bar{#1}}}
\newcommand{\Mat}[1]{\mathbf{#1}}

\newcommand{\Norm}[1]{\lVert#1\rVert}
\newcommand{\InProd}[2]{\langle#1,#2\rangle}

\jmlrheading{\null}{\null}{\null}{\null}{\null}{\null}{%
  Saibal De, Hadi Salehi and Alex Gorodetsky
}

\ShortHeadings{%
  Symmetry Breaking in Bayesian Matrix Factorization
}{%
  De, Salehi and Gorodetsky
}
\firstpageno{1}

\begin{document}

\title{%
  Efficient MCMC Sampling for Bayesian Matrix Factorization by Breaking
  Posterior Symmetries
}

\author{%
  \name Saibal De
  \email saibalde@umich.edu\\
  \addr Department of Mathematics\\
        University of Michigan\\
        Ann Arbor, MI 48109, USA
  \AND
  \name Hadi Salehi
  \email hsalehi@umich.edu\\
  \addr Department of Aerospace Engineering\\
        University of Michigan\\
        Ann Arbor, MI 48109, USA
  \AND
  \name Alex Gorodetsky
  \email goroda@umich.edu\\
  \addr Department of Aerospace Engineering\\
        University of Michigan\\
        Ann Arbor, MI 48109, USA
}

\editor{\null}

\maketitle

\begin{abstract}%
  Bayesian low-rank matrix factorization techniques have become an essential
  tool for relational data analysis and matrix completion. A standard approach
  is to assign zero-mean Gaussian priors on the columns or rows of factor
  matrices to create a conjugate system. This choice of prior leads to simple
  implementations; however it also  causes symmetries in the posterior
  distribution that can severely reduce the efficiency of Markov-chain
  Monte-Carlo (MCMC) sampling approaches. In this paper, we propose a simple
  modification to the prior choice that \emph{provably} breaks these symmetries
  and maintains/improves accuracy.  Specifically, we provide conditions that the
  Gaussian prior mean and covariance must satisfy so the posterior does not
  exhibit invariances that yield sampling difficulties.  For example, we show
  that using non-zero linearly independent prior means significantly lowers the
  autocorrelation of MCMC samples, and can also lead to lower reconstruction
  errors.
\end{abstract}

\begin{keywords}
  Matrix completion,
  low-rank matrix factorization,
  Bayesian inference,
  posterior symmetry breaking,
  Markov chain Monte-Carlo (MCMC)
\end{keywords}

\section{Introduction}
\label{sec:introduction}

The \emph{matrix completion} problem seeks to use partial observations of a
matrix to estimate missing entries. Formally, let $\Mat{X} \in \RR^{m \times n}$
be our target data matrix, and let $\Lambda \subseteq \{1, \ldots, m\} \times
\{1, \ldots, n\}$ be a set of matrix indices where the matrix element is
observed:
\begin{equation*}
  \Vec{y} = \CP_\Lambda(\Mat{X}) + \Vec{\eta}.
\end{equation*}
Here, $\CP_\Lambda : \RR^{m \times n} \to \RR^{\Card{\Lambda}}$,
$\Card{\Lambda}$ being the cardinality of the set $\Lambda$, is the linear
projection map
\begin{equation*}
  \CP_\Lambda(\Mat{X}) = (x_{\Vec{\lambda}} : \Vec{\lambda} = (\lambda_1,
  \lambda_2) \in \Lambda),
\end{equation*}
and $\Vec{\eta} \in \RR^{\Card{\Lambda}}$ is a vector of additive noises. Our
goal is then to recover the matrix $\Mat{X}$ from the observations $\Vec{y}$.
This problem commonly arises in many practical applications such as recommender
system design \citep{takacs2008investigation}, drug-target interaction
prediction \citep{yamanishi2010drug, zheng2013collaborative}, image inpainting
\citep{he2015total, li2020rank}, social network topology recovery
\citep{mahindre2019sampling} and sensor localization \citep{xue2019locating}.

\subsection{Related Works}

This matrix problem is naturally ill-posed, and obtaining robust and accurate
solutions requires imposing additional \textit{regularity} conditions on the
underlying data matrix such as sparsity or low-rank structures. In this paper,
we consider the problem of low-rank matrix completion, where we might attempt to
recover the matrix by nuclear norm minimization
\begin{equation}
  \label{eq:matrix-recovery-nuclear}
  \Minimize_{\Mat{X} \in \RR^{m \times n}} \quad \Norm{\Mat{X}}_* \quad
  \text{subject to} \quad \lVert \Vec{y} - \CP_{\Lambda}(\Mat{X}) \rVert \leq
  \delta,
\end{equation}
for some constant $\delta \geq 0$ that depends on the level of noise. Note that
the nuclear norm is the convex relaxation of rank of a matrix
\citep{candes2009exact}, hence the objective in the optimization problem
naturally encourages low-rank structure of the reconstruction. In
\citet{candes2009exact}, \citet{candes2010power} and \citet{candes2010matrix}
the authors establish that, under mild assumptions about certain incoherence
properties of $\Mat{X}$, solving \eqref{eq:matrix-recovery-nuclear} leads to
accurate recovery of the underlying data matrix with with surprisingly few
observations. In fact, in absence of noise (that is, when $\Vec{\eta} = \Vec{0}$
and $\delta = 0$) exact recovery is possible with high probability.

The standard solution method for optimization \eqref{eq:matrix-recovery-nuclear}
is semi-definite programming, which is expensive when the matrix sizes $m$ and
$n$ are large. In this setup, it is advantageous to explicitly use the low-rank
factorization
\begin{equation}
  \label{eq:factor-model}
  \Mat{X} = \Mat{A} \Mat{B}^\top, \quad \Mat{A} \in \RR^{m \times r}, \quad
  \Mat{B} \in \RR^{n \times r}, \quad r \ll \min\{m, n\},
\end{equation}
and solve the optimization problem
\begin{equation}
  \label{eq:matrix-factor-recovery-frobenius}
  \Minimize_{\Mat{A} \in \RR^{m \times r}, \Mat{B} \in \RR^{n \times r}} \quad
  \Norm{\Mat{A}}_F^2 + \Norm{\Mat{B}}_F^2 \quad \text{subject to} \quad \lVert
  \Vec{y} - \CP_{\Lambda}(\Mat{A} \Mat{B}^\top) \rVert \leq \delta,
\end{equation}
We can show that the optimization problems \eqref{eq:matrix-recovery-nuclear}
and \eqref{eq:matrix-factor-recovery-frobenius} are equivalent as long as the
estimated rank $r$ is chosen to be larger than the true rank
\citep{recht2010guaranteed}. We also consider the unconstrained version of
\eqref{eq:matrix-factor-recovery-frobenius}:
\begin{equation}
  \label{eq:optimization}
  \Minimize_{\Mat{A} \in \RR^{m \times r}, \Mat{B} \in \RR^{n \times r}}
  \Norm{\Vec{y} - \mathcal{P}_\Lambda(\Mat{A} \Mat{B}^\top)}_2^2 +
  \frac{\omega}{2} \Norm{\Mat{A}}_F^2 + \frac{\omega}{2} \Norm{\Mat{B}}_F^2.
\end{equation}
Generally speaking, these optimization-based approaches attempt to minimize the
distance between the observed entries and their corresponding low-rank
predictions while regularizing over the low-rank factors, and they serve as the
base of a wide class of methods for low-rank matrix recovery
\citep{srebro2005maximum, mnih2008probabilistic, davenport2016overview}.

In this paper we instead focus on probabilistic approaches for which the
solution also quantifies uncertainty in the predictions. Within this context the
objective function \eqref{eq:optimization} can easily be viewed as the negative
log-posterior in a Bayesian parameter estimation problem with data $\Vec{y}$ and
parameters $\Mat{A}$ and $\Mat{B}$---it corresponds to i.i.d.\ zero-mean
Gaussian priors on the entries of the factor matrices and observations corrupted
by Gaussian noise. \citet{lim2007variational} and \citet{raiko2007principal}
each apply variational Bayes approximations of this inference model to analyze
the Netflix prize challenge \citep{bennett2007netflix} to great success.
Further, \citet{nakajima2011theoretical} and \citet{nakajima2013global} develop
a theoretical framework to analyze the variational Bayes low-rank matrix
factorization.

A fundamental issue with the variational Bayes approach lies in one of its
modeling assumptions---namely the factors $\Mat{A}$ and $\Mat{B}$ are taken to
be independent. \citet{salakhutdinov2008bayesian} argues a fully Bayesian
framework, which does not need this assumption, can outperform the variational
models. Their proposed Markov-chain Monte-Carlo (MCMC) based approach is used in
several later works \citep{chen2014stochastic, ahn2015large}. More recently, it
has been adapted for recovering low-rank tensor factorizations as well
\citep{rai2014scalable, zhao2015bayesian1, zhao2015bayesian2}.

\subsection{Our Contributions}

In most of these variational and fully Bayesian inference setups, the priors and
observations models are the same---they use Gaussian priors with zero means on
the columns of the factor matrices, and assume observations are corrupted by
additive Gaussian noise. One of the reason for the popularity of this choice of
priors is rooted in the fact that it leads to simple analytical conditional
posteriors, which is ideal for devising a Gibbs MCMC sampler
\citep{alquier2014bayesian, alquier2015bayesian}. Unfortunately, this choice of
priors cannot fully mitigate the non-identifiability of the low-rank
factorization \eqref{eq:factor-model}---given any $r \times r$ non-singular
matrix $\Mat{W}$ we can construct the new factors
\begin{equation*}
  \Mat{\tilde{A}} = \Mat{A} \Mat{W} \text{ and } \Mat{\tilde{B}} = \Mat{B}
  \Mat{W}^{-\top} \implies \Mat{X} = \Mat{\tilde{A}} \Mat{\tilde{B}}^\top.
\end{equation*}
The effect of this \emph{invertible invariance} shows up in the Bayes posteriors
in the form of symmetries, as illustrated in Figure~\ref{fig:fig2}, where we
plot the joint posterior between various components of the factor matrices,
constructed from Hamiltonian Monte-Carlo (HMC) samples, for the fully observed
$4 \times 4$ rank-2 matrix described in Example~1
(Section~\ref{sec:numerical-results}).  Effective MCMC sampling from
distributions with such wide varying multi-modal and non-connected geometries
is, in general, a difficult task.

\begin{figure}
  \centering

  \includegraphics[trim=60 350 60 220, clip, width=\textwidth]{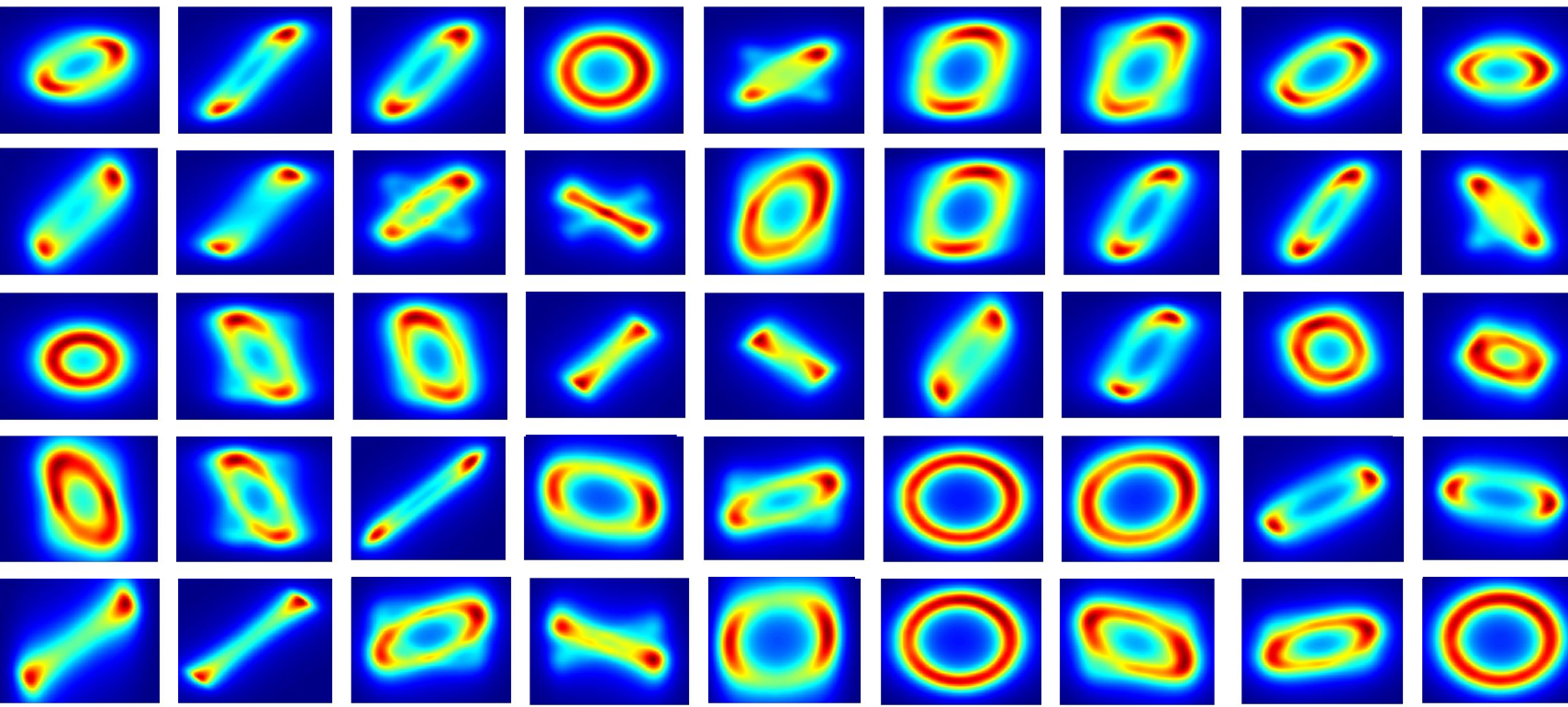}
  \caption{Joint posterior between some components of the factor matrices.
  Results obtained using Hamiltonian Monte Carlo with zero mean priors.}
  \label{fig:fig2}

  \includegraphics[trim=70 360 60 210, clip, width=\textwidth]{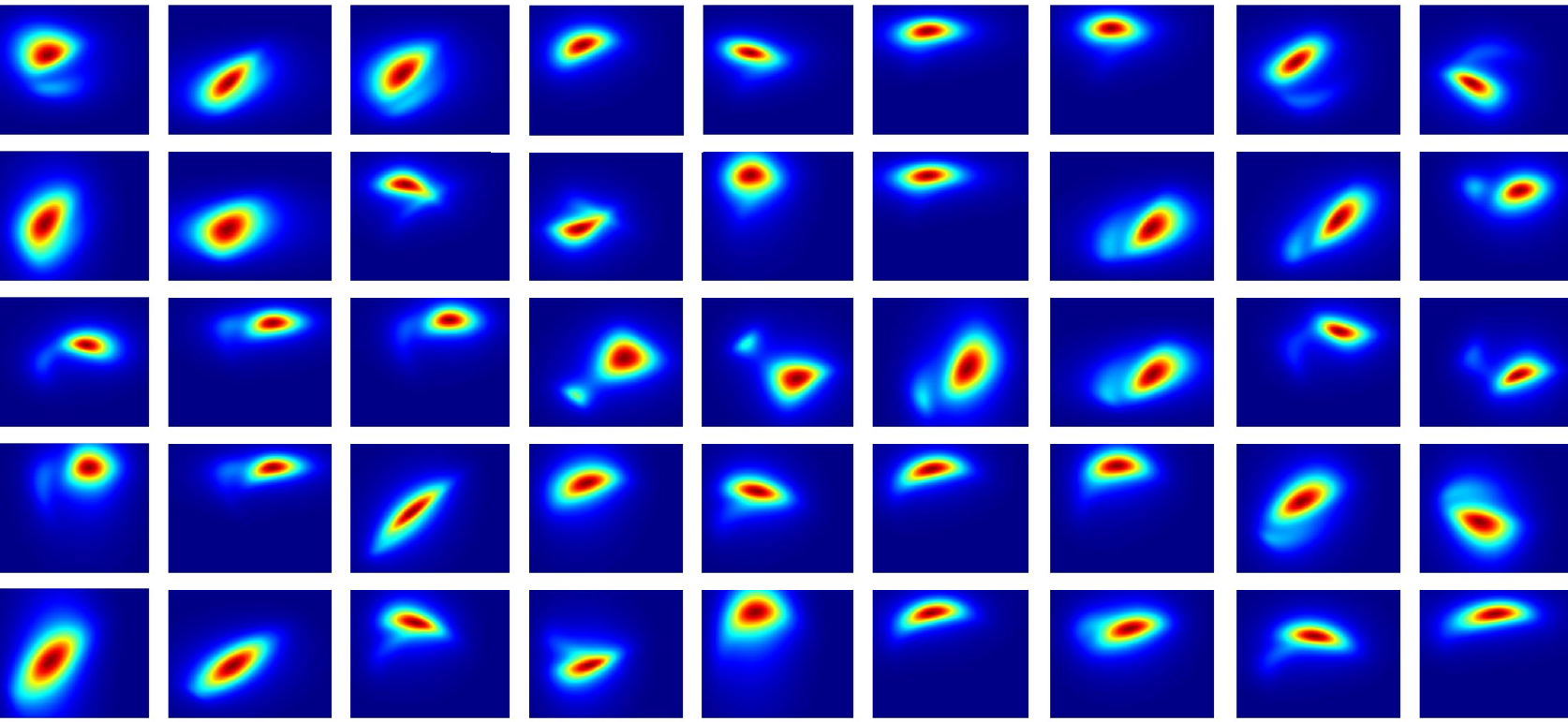}
  \caption{Joint posterior between some components of the factor matrices (same
  as in Figure~\ref{fig:fig2}). Results obtained using Hamiltonian Monte Carlo,
  but this time with non-zero mean priors.}
  \label{fig:fig3}
\end{figure}

The main goal of this paper is to advocate a simple change in the prior
specification, which does not break the local conjugacy required for Gibbs
samplers, but breaks the posterior symmetries. Our contributions are threefold:
\begin{itemize}
\item
  In Theorem~\ref{thm:posterior-symmetry}, we derive the exact conditions on the
  matrix $\Mat{W}$ such that the posterior, assuming zero-mean Gaussian priors, is
  invariant under the transformation $(\Mat{A}, \Mat{B}) \mapsto (\Mat{A}
  \Mat{W}, \Mat{B} \Mat{W}^{-\top})$.
\item
  In Corollary~\ref{cor:symmetry-breaking-sufficient} we prove that choosing a
  system of linearly independent prior means is sufficient for breaking this
  posterior symmetry. An illustration is given in Figure~\ref{fig:fig3}, where
  we plot the same set of joint posteriors as in Figure~\ref{fig:fig2}, but now
  constructed from HMC samples obtained with non-zero mean priors.
\item
  By breaking the symmetry via introducing non-zero mean priors, we counter the
  non-identifiability of low-rank factorizations. We demonstrate this ultimately
  leads to better performance for MCMC sampling algorithms on matrices
  constructed from both synthetic and real-world data. We observe up to an order
  of magnitude decrease in the autocorrelations of generated samples and
  corresponding improvement in reconstruction errors of the underlying data
  matrices in our numerical examples in Section~\ref{sec:numerical-results}.
\end{itemize}

Note that our proposed change in prior means is very simple to implement---one
needs to change at most a few lines of code in any existing applications.
Constructing the appropriate prior means is also straightforward. From random
matrix theory, we know that if the entries of a tall-and-thin matrix are sampled
i.i.d.\ from \emph{any} continuous random variable in $\RR$, then the columns of
the matrix form a linearly independent system with probability 1.

The rest of the paper is structured as follows. In
Section~\ref{sec:bayesian-setup}, we formally introduce the Bayesian inference
setup and quantify the symmetries arising from zero-mean Gaussian priors. In
Section~\ref{sec:symmetry-breaking}, we show that choosing the priors means to
be non-zero in a systematic fashion breaks the invertible invariance. In
Section~\ref{sec:numerical-results}, we present the numerical experiments.
Finally, in Section~\ref{sec:conclusion}, we present our concluding remarks and
some directions for future work.

\section{Notations and Bayesian Inference Setup}
\label{sec:bayesian-setup}

In this section, we introduce the notations we use throughout the rest of the
paper, and introduce our Bayesian inference setup for the low-rank matrix
factorization problem.

\subsection{Notations}
\label{sec:notation}

A vector $\Vec{x}$ is always represented as a column, a row-vector is
represented as $\Vec{x}^\top$. The vector $\Vec{e}_i$ is the $i$-th standard
basis of appropriate (and inferable from context) size---all but its $i$-th
entries are zero, and the non-zero entry is one.

Given a matrix $\Mat{X} \in \RR^{m \times n}$, we use $x_{ij}$ to denote its
$(i, j)$-th entry. Additionally, $\RowVec{x}_i \in \RR^{n}$ and $\Vec{x}_j \in
\RR^m$ denotes $i$-th row and the $j$-th column of the matrix; thus
\begin{equation*}
  \Mat{X} =
  \begin{bmatrix}
    x_{11} & \cdots & x_{1n} \\
    \vdots & \ddots & \vdots \\
    x_{m1} & \cdots & x_{mn}
  \end{bmatrix}
  =
  \begin{bmatrix}
    \Vec{x}_1 & \cdots & \Vec{x}_n
  \end{bmatrix}
  =
  \begin{bmatrix}
    \RowVec{x}_1^\top \\
    \vdots \\
    \RowVec{x}_m^\top
  \end{bmatrix}.
\end{equation*}
The matrix $\Mat{I}_n$ denotes the $n \times n$ identity matrix.

We adopt MATLAB's notation for indexing---given index sets $\Lambda_r \subseteq
\{1, \ldots, m\}$ and $\Lambda_c \subseteq \{1, \ldots, n\}$ we use
$\Mat{X}[\Lambda_r, \Lambda_c]$ to denote the intersection of the rows of
$\Mat{X} \in \RR^{m \times n}$ indexed by $\Lambda_r$ and columns indexed by
$\Lambda_c$. We also adopt the following slight abuses of notation:
\begin{align*}
  \Mat{X}[i, \Lambda_c] &= \Mat{X}[\{i\}, \Lambda_c], & \Mat{X}[:, \Lambda_c] &=
  \Mat{X}[\{1, \ldots, m\}, \Lambda_c], \\
  \Mat{X}[\Lambda_r, j] &= \Mat{X}[\Lambda_r, \{j\}], & \Mat{X}[\Lambda_r, :] &=
  \Mat{X}[\Lambda_r, \{1, \ldots, n\}].
\end{align*}

\subsection{Prior and Likelihood Models}
\label{sec:probability-model}

Given a low-rank factorization \eqref{eq:factor-model}, we impose
independent Gaussian priors on the columns of the factor matrices
\begin{equation*}
  \Vec{a}_k \sim \NormalDist(\Vec{\mu}_{a,k}, \tau_{a,k}^{-1} \Mat{I}_m), \quad
  \Vec{b}_k \sim \NormalDist(\Vec{\mu}_{b,k}, \tau_{b,k}^{-1} \Mat{I}_n).
\end{equation*}
The joint prior is then given by
\begin{equation*}
  p(\Mat{A}, \Mat{B}) = \prod_{k = 1}^r \NormalDist(\Vec{a}_k \mid
  \Vec{\mu}_{a,k}, \tau_{a,k}^{-1} \Mat{I}_m) \NormalDist(\Vec{b}_k \mid
  \Vec{\mu}_{b,k}, \tau_{b,k}^{-1} \Mat{I}_n).
\end{equation*}
We assume real-valued matrices and a standard additive Gaussian noise model
\begin{equation*}
  y_{\Vec{\lambda}} = x_{\Vec{\lambda}} + \eta_{\Vec{\lambda}}, \quad
  \eta_{\Vec{\lambda}} \sim \NormalDist(0, \tau_\eta^{-1}), \quad \Vec{\lambda}
  \in \Lambda.
\end{equation*}
Here $\Lambda$ is the set of location indices where the matrix was observed.
The corresponding likelihood is given by
\begin{equation}
  \label{eq:likelihood}
  p(\Vec{y} \mid \Mat{A}, \Mat{B}) = \prod_{\Vec{\lambda} \in \Lambda}
  \NormalDist(y_{\Vec{\lambda}} \mid \RowVec{a}_{\lambda_1}^\top
  \RowVec{b}_{\lambda_2}, \tau_\eta^{-1}).
\end{equation}
The effect of invertible invariance of the likelihood is immediately obvious:

\begin{proposition}
  \label{prop:likelihood-invariance}
  The likelihood $p(\Vec{y} \mid \Mat{A}, \Mat{B})$ defined in
  \eqref{eq:likelihood} is invariant under invertible transformations. In
  particular, if $\Mat{W} \in \RR^{r \times r}$ is an invertible matrix, then
  \begin{equation*}
    p(\Vec{y} \mid \Mat{A}, \Mat{B}) = p(\Vec{y} \mid \Mat{A} \Mat{W}, \Mat{B}
    \Mat{W}^{-\top}) \quad \text{for all} \quad \Mat{A} \in \RR^{m \times r}
    \text{ and } \Mat{B} \in \RR^{n \times r}.
  \end{equation*}
\end{proposition}

\begin{proof}
  Let $\Mat{\tilde{A}} = \Mat{A} \Mat{W}$ and $\Mat{\tilde{B}} = \Mat{B}
  \Mat{W}^{-\top}$. Then we note
  \begin{equation*}
    \Mat{\tilde{A}} \Mat{\tilde{B}}^\top = (\Mat{A} \Mat{W}) (\Mat{B}
    \Mat{W}^{-\top})^\top = \Mat{A} \Mat{W} \Mat{W}^{-1} \Mat{B}^\top = \Mat{A}
    \Mat{B}^\top.
  \end{equation*}
  Since all the matrix entries appearing in \eqref{eq:likelihood} are the
  same for both $(\Mat{A}, \Mat{B})$ and $(\Mat{\tilde{A}}, \Mat{\tilde{B}})$,
  it follows that the likelihood is invariant.
\end{proof}

\subsection{The Posterior and its Symmetries}
\label{sec:posterior-symmetry}

Using Bayes rule, the posterior is
\begin{equation*}
  p(\Mat{A}, \Mat{B} \mid \Vec{y}) \propto p(\Mat{A}, \Mat{B}) p(\Vec{y} \mid
  \Mat{A}, \Mat{B}) = p(\Mat{A}) p(\Mat{B}) p(\Vec{y} \mid \Mat{A}, \Mat{B}).
\end{equation*}
The negative log posterior is therefore
\begin{equation}
  \label{eq:posterior}
  \begin{split}
    -\ln p(\Mat{A}, \Mat{B} \mid \Vec{y})
    &= \frac{m r}{2} \ln (2\pi) + \frac{m}{2} \sum_{k = 1}^r \ln
    \frac{1}{\tau_{a, k}} + \frac{1}{2} \sum_{k = 1}^r \tau_{a,k}
    \Norm{\Vec{a}_k - \Vec{\mu}_{a,k}}^2 \\
    &+ \frac{n r}{2} \ln (2\pi) + \frac{n}{2} \sum_{k = 1}^r \ln
    \frac{1}{\tau_{b, k}} + \frac{1}{2} \sum_{k = 1}^r \tau_{b,k}
    \Norm{\Vec{b}_k - \Vec{\mu}_{b,k}}^2 \\
    &+ \frac{\Card{\Lambda}}{2} \ln (2\pi) + \frac{\Card{\Lambda}}{2} \ln
    \frac{1}{\tau_\eta} + \frac{\tau_\eta}{2} \sum_{\Vec{\lambda} \in \Lambda}
    (y_{\Vec{\lambda}} - \RowVec{a}_{\lambda_1}^\top \RowVec{b}_{\lambda_2})^2 +
    \text{const.}
  \end{split}
\end{equation}
The first two lines correspond to the priors on $\Mat{A}$ and $\Mat{B}$, and
the first three terms of the final line to the likelihood. The constant
term corresponds to the evidence $p(\Vec{y})$ of the observations, and can
generally be ignored. We have already seen that the likelihood term is invariant
under invertible transformations of the form $(\Mat{A}, \Mat{B}) \mapsto
(\Mat{A} \Mat{W}, \Mat{B} \Mat{W}^{-\top})$.  Now we investigate any effect the
prior terms might have. We immediately note the following:

\begin{proposition}
  \label{prop:individual-invariance}
  The posterior corresponding to \eqref{eq:posterior} is invariant
  under invertible transformation $\Mat{W} \in \RR^{r \times r}$, i.e.
  \begin{equation*}
    p(\Mat{A}, \Mat{B} \mid \Vec{y}) = p(\Mat{A} \Mat{W}, \Mat{B}
    \Mat{W}^{-\top} \mid \Vec{y}) \quad \text{for all} \quad \Mat{A} \in \RR^{m
    \times r} \text{ and } \Mat{B} \in \RR^{n \times r},
  \end{equation*}
  if and only if the terms
  \begin{align*}
    f_1(\Mat{A}) &= \sum_{k = 1}^r \tau_{a,k} \Norm{\Vec{a}_k}^2,
                 &
    f_2(\Mat{A}) &= \sum_{k = 1}^r \tau_{a,k} \Vec{\mu}_{a,k}^\top \Vec{a}_k, \\
    f_3(\Mat{B}) &= \sum_{k = 1}^r \tau_{b,k} \Norm{\Vec{b}_k}^2,
                 &
    f_4(\Mat{B}) &= \sum_{k = 1}^r \tau_{b,k} \Vec{\mu}_{b,k}^\top \Vec{b}_k
  \end{align*}
  are individually invariant under the $\Mat{A} \mapsto \Mat{A} \Mat{W}$ and
  $\Mat{B} \mapsto \Mat{B} \Mat{W}^{-\top}$ transformations.
\end{proposition}

We can establish this result by essentially using the homogeneity of the $f_1$,
$f_2$, $f_3$ and $f_4$ terms (a formal proof is presented in
Appendix~\ref{app:individual-invariance}). Clearly, the addition of prior
imposes further restrictions on $\Mat{W}$ for invertible invariance of the
posterior (compared to the invariance of the likelihood). In particular, with
zero-mean priors, the invariance exhibited by the likelihood under the
transformation $(\Mat{A}, \Mat{B}) \mapsto (\Mat{A}\Mat{W}, \Mat{B}
\Mat{W}^{-\top})$ holds for the posterior only when $\Mat{W}$ is restricted to a
very particular subclass of invertible matrices:

\begin{theorem}[Posterior symmetries with zero mean priors]
  \label{thm:posterior-symmetry}
  Let $\Vec{\mu}_{a,k} = \Vec{0}$ and $\Vec{\mu}_{b,k} = \Vec{0}$ for all $1
  \leq k \leq r$ and denote the diagonal matrices of the precision of the priors
  on the columns as
  \begin{equation*}
    \Mat{T}_a = \Diag(\tau_{a, 1}, \ldots, \tau_{a, r}), \quad \Mat{T}_b =
    \Diag(\tau_{b, 1}, \ldots, \tau_{b, r}).
  \end{equation*}
  Let $\{\Lambda_1, \ldots, \Lambda_q\}$ be a partition of $\{1, \ldots, r\}$
  defined by the following:
  \begin{equation*}
    k, k' \in \Lambda_\ell \iff \tau_{a, k} \tau_{b, k} = \tau_{a, k'} \tau_{b,
    k'}.
  \end{equation*}
  Then the posterior corresponding to \eqref{eq:posterior} is invariant
  under the $(\Mat{A}, \Mat{B}) \mapsto (\Mat{A} \Mat{W}, \Mat{B}
  \Mat{W}^{-\top})$ transformation with invertible $\Mat{W} \in \RR^{r \times
  r}$ if and only if we can decompose
  \begin{equation*}
    \Mat{W} = \Mat{T}_a^{1/2} \Mat{Q} \Mat{T}_a^{-1/2} = \Mat{T}_b^{-1/2}
    \Mat{Q} \Mat{T}_b^{1/2},
  \end{equation*}
  where $\Mat{Q}$ is orthogonal and block diagonal w.r.t.\ the partition
  $\{\Lambda_1, \ldots, \Lambda_q\}$, i.e. the sub-matrices
  \begin{equation*}
    \Mat{Q}[\Lambda_{\ell_1}, \Lambda_{\ell_2}] \text{ are }
    \begin{cases}
      \text{orthogonal} & \text{ if } \ell_1 = \ell_2 \\
      \text{zero} & \text{ if } \ell_1 \neq \ell_2
    \end{cases}.
  \end{equation*}
\end{theorem}

We note that one direction of this theorem (that the above form of $\Mat{W}$ is
sufficient for the invariance of posterior) appears in
\citet{nakajima2011theoretical}. We claim this structure of the matrix $\Mat{W}$
is also necessary for invertible invariance as well; a formal proof is presented
in Appendix~\ref{app:posterior-symmetry}.

Two extreme cases of invariance with zero mean priors can be derived immediately
from this theorem:

\begin{corollary}
  Let $\Vec{\mu}_{a,k} = \Vec{0}$ and $\Vec{\mu}_{b,k} = 0$, $1 \leq k \leq r$.
  Further suppose $\tau_{a,1} = \cdots = \tau_{a,r}$ and $\tau_{b,1} = \cdots =
  \tau_{b,r}$. Then the posterior corresponding to
  \eqref{eq:posterior} is invariant under invertible transformation
  $\Mat{W}$ if and only if $\Mat{W}$ is orthogonal.
\end{corollary}

\begin{proof}
  We have $\tau_{a, 1} \tau_{b, 1} = \cdots = \tau_{a, r} \tau_{b, r}$. Hence
  $\Mat{Q}$ only has one block. Consequently invertible invariance holds if and
  only if $\Mat{Q}$ is orthogonal (by Theorem~\ref{thm:posterior-symmetry}).
  Additionally, we can write $\Mat{T}_a = \tau_{a, 1} \Mat{I}_r$ where
  $\Mat{I}_r$ is the $r \times r$ identity matrix. It follows that
  \begin{equation*}
    \Mat{W} = \Mat{T}_a^{1/2} \Mat{Q} \Mat{T}_a^{-1/2} = \tau_{a, 1}^{1/2}
    \Mat{I}_r \Mat{Q} \tau_{a, 1}^{-1/2} \Mat{I}_r = \Mat{Q},
  \end{equation*}
  i.e.\ $\Mat{Q}$ is orthogonal if and only if $\Mat{W}$ is also orthogonal.
\end{proof}

\begin{corollary}
  Let $\Vec{\mu}_{a,k} = \Vec{0}$ and $\Vec{\mu}_{b,k} = 0$, $1 \leq k \leq r$.
  Further suppose $\tau_{a,1} \tau_{b,1}, \ldots, \tau_{a,r} \tau_{b, r}$ are
  all distinct. Then the posterior corresponding to
  \eqref{eq:posterior} is invariant under invertible transformation
  $\Mat{W}$ if and only if $\Mat{W}$ is diagonal with non-zero entries $\pm 1$.
\end{corollary}

\begin{proof}
  It follows from Theorem~\ref{thm:posterior-symmetry} that $\Mat{Q}$ must be
  block diagonal with block size $1$, and the each block has to be $\pm 1$
  (these are the only $1 \times 1$ orthogonal matrices). Let us write $\Mat{Q} =
  \Diag(q_1, \ldots, q_r)$ with each $q_i = \pm 1$. Then we have
  \begin{equation*}
    \Mat{W} = \Diag(\tau_{a, 1}^{1/2}, \ldots, \tau_{a, r}^{1/2}) \Diag(q_1,
    \ldots, q_r) \Diag(\tau_{a, 1}^{-1/2}, \ldots, \tau_{a, r}^{-1/2}) =
    \Diag(q_1, \ldots, q_r).
  \end{equation*}
  It follows that invertible invariance holds if and only if $\Mat{W}$ is
  diagonal with entries $\pm 1$.
\end{proof}

These corollaries make it clear that under zero-mean priors, there are at the
very least $2^r$ symmetries in the posterior, corresponding to the $\Mat{W} =
\Diag(\pm 1, \ldots, \pm 1)$ transformation matrices.

\section{Breaking the Symmetries with Non-Zero Mean Priors}
\label{sec:symmetry-breaking}

It is clear from Proposition~\ref{prop:individual-invariance} and
Theorem~\ref{thm:posterior-symmetry} that the secret to
completely breaking the symmetries can only hide in the $f_2$ and $f_4$ terms,
which involve the prior means. This leads to our main result:

\begin{theorem}[Breaking posterior symmetries]
  \label{thm:symmetry-breaking}
  Let $\Mat{T}_a$, $\Mat{T}_b$ and $\{\Lambda_1, \ldots, \Lambda_q\}$ be as
  defined in the statement of Theorem~\ref{thm:posterior-symmetry}. Define the
  prior mean matrices
  \begin{equation*}
    \Mat{M}_a =
    \begin{bmatrix}
      \Vec{\mu}_{a, 1} & \cdots & \Vec{\mu}_{a, r}
    \end{bmatrix}
    \quad \text{and} \quad
    \Mat{M}_b =
    \begin{bmatrix}
      \Vec{\mu}_{b, 1} & \cdots & \Vec{\mu}_{b, r}
    \end{bmatrix}.
  \end{equation*}
  Then the posterior $p(\Mat{A}, \Mat{B} \mid \Vec{y})$ is not invariant under
  the $(\Mat{A}, \Mat{B}) \mapsto (\Mat{A} \Mat{W}, \Mat{B} \Mat{W}^{-\top})$
  transformation for any non-identity invertible $r \times r$ matrix $\Mat{W}$
  if and only if the matrices
  \begin{equation}
    \label{eq:symmetry-breaking-cond}
    \begin{bmatrix}
      \Mat{M}_a[:, \Lambda_\ell] \Mat{T}_a[\Lambda_\ell, \Lambda_\ell]^{1/2} \\
      \Mat{M}_b[:, \Lambda_\ell] \Mat{T}_b[\Lambda_\ell, \Lambda_\ell]^{1/2}
    \end{bmatrix}
  \end{equation}
  have full column rank for all $1 \leq \ell \leq q$.
\end{theorem}

We postpone a formal proof until Appendix~\ref{app:symmetry-breaking}; for now,
we present an immediate corollary: it provides an extremely simple way to ensure
full rank matrices \eqref{eq:symmetry-breaking-cond}. It essentially
states that if the means are chosen to be linearly independent (e.g. the entries
of the mean matrix are sampled i.i.d.\ from a zero-mean Gaussian
distribution---the columns would be independent with probability 1), then the
invariance is broken.

\begin{corollary}
  \label{cor:symmetry-breaking-sufficient}
  Suppose either $\{\Vec{\mu}_{a,k} : 1 \leq k \leq r\}$ or $\{\Vec{\mu}_{b,k} :
  1 \leq k \leq r\}$ form a linearly independent set in $\RR^m$ or $\RR^n$. Then
  the posterior $p(\Mat{A}, \Mat{B} \mid \Vec{y})$ is not invariant under
  any non-identity invertible transformations.
\end{corollary}

\begin{proof}
  Suppose $\{\Vec{\mu}_{a,k} : 1 \leq k \leq r\}$ is a linearly independent set
  in $\RR^m$. Then $\Mat{M}_a[:, \Lambda_\ell]$, and consequently $\Mat{M}_a[:,
  \Lambda_\ell] \Mat{T}_a[\Lambda_\ell, \Lambda_\ell]^{1/2}$, are full rank for
  all $\ell$.  It follows that
  \begin{equation*}
    \Mat{P}_\ell =
    \begin{bmatrix}
      \Mat{M}_a[:, \Lambda_\ell] \Mat{T}_a[\Lambda_\ell, \Lambda_\ell]^{1/2} \\
      \Mat{M}_b[:, \Lambda_\ell] \Mat{T}_b[\Lambda_\ell, \Lambda_\ell]^{1/2}
    \end{bmatrix}
  \end{equation*}
  has full rank for all $\ell$, and Theorem~\ref{thm:symmetry-breaking} applies.
  A similar proof can be constructed when the prior means $\{\Vec{\mu}_{b,k} : 1
  \leq k \leq r\}$ form a linearly independent set in $\RR^n$.
\end{proof}

Thus, by carefully choosing non-zero means for the priors on matrix factors, we
can ensure that for any $r \times r$ invertible matrix $\Mat{W} \neq \Mat{I}$,
the posteriors $p(\Mat{A}, \Mat{B} \mid \Vec{y})$ and $p(\Mat{A} \Mat{W}^\top,
\Mat{B} \Mat{W}^{-\top} \mid \Vec{y})$ are distinct. In other words, with this
choice of prior distributions, we can keep the identifiability issue from
affecting Bayesian inference.

\section{Numerical Results}
\label{sec:numerical-results}

We now demonstrate that symmetry breaking improves MCMC sampling, both in terms
of efficiency (by decreasing autocorrelation) and accuracy (by reducing
reconstruction error), with four numerical experiments. Examples~1 and~2 apply
Bayesian matrix factorization with synthetic data, and Examples~3 and~4 work
with real-world data. For each example, the entries of the non-zero prior mean
matrices are sampled from a uniform distribution. We use the root mean squared
error (RMSE) as a measure of error between the true matrix $\Mat{X} \in \RR^{m
\times n}$ and its reconstruction $\hat{\Mat{X}} \in \RR^{m \times n}$:
\begin{equation*}
  \text{RMSE} = \sqrt{\frac{1}{mn} \sum_{i = 1}^m \sum_{j = 1}^n (x_{ij} -
  \hat{x}_{ij})^2}.
\end{equation*}

\subsection{Example 1: Fully Observed Synthetic Matrix}

We contrast the results obtained from running the HMC and Gibbs samplers on a
rank-2 4$\times$4 matrix
\begin{equation*}
  \Mat{X} =
  \begin{bmatrix}
    1 &  0 & 1 &  5 \\
    2 & -1 & 1 &  4 \\
    4 & -1 & 3 & 14 \\
    3 & -1 & 2 &  9
  \end{bmatrix},
\end{equation*}
with zero and non-zero mean priors; the non-zero prior means $\Mat{M}_a$ and
$\Mat{M}_b$ are constructed by sampling each of the entries from the
$\text{Uniform}(0, 1)$ distribution. We observe the full matrix with noise
precision $\tau_\eta = 10^4$, and MCMC results are obtained using both Gibbs and
HMC samplers that use 10 chains, each with 20000 samples. We also assume that the
precision is unknown, and follow the standard procedure of hierarchical Bayes by
inferring the precision with prior $\tau_\eta \sim \text{Ga}(3, 10^{-2})$.  This
leads to a conditionally conjugate posterior on $\tau_{\eta}$
\citep{alquier2014bayesian}.\footnote{Even though our theory was constructed
assuming $\tau_\eta$ is constant, it generalizes in a very straightforward
manner for arbitrary prior on $\tau_\eta$.} In Figure~\ref{fig:fig2} and
Figure~\ref{fig:fig3}, we plot various joint posteriors, obtained using the
samples generated using HMC, corresponding to the zero and non-zero mean priors.
We can clearly see that the symmetries of the posterior in Figure~\ref{fig:fig2}
(corresponding to the zero mean priors) are not observed when using non-zero
mean priors in Figure~\ref{fig:fig3}.  This symmetry-breaking leads to better
performance for MCMC samplers. For example, in Figure~\ref{fig:fig1}, we plot
the autocorrelations for factor $a_{1,1}$ of the samples generated from both HMC
and Gibbs samplers. The autocorrelations are significantly lower, for both
samplers, when non-zero mean priors are used.  This indicates using non-zero
mean priors leads to better/faster mixing and smaller
autocorrelations---regardless of algorithm used because the geometry of the
posteriors are more favorable to the types of structures exploited by virtually
all MCMC techniques.  Further, in Figure~\ref{fig:fig4}, we plot the histogram
of RMSE values of the matrix reconstruction for 1000 repetitions of the data
gathering procedure---each experiment differs due to the noise realization and
randomly sampled prior mean. For each of the repetitions we run 10 chains of the
Gibbs sampler. Here the RMSE is on the order of $10^{-2}$ for both
samplers---this is exactly what we expected because it is the order of
the noise standard deviation. We can clearly see that, for this example, there
is little difference in the RMSE based on the prior mean.

\begin{figure}
  \centering
  \begin{subfigure}{0.49\textwidth}
    \includegraphics[width=0.9\linewidth]{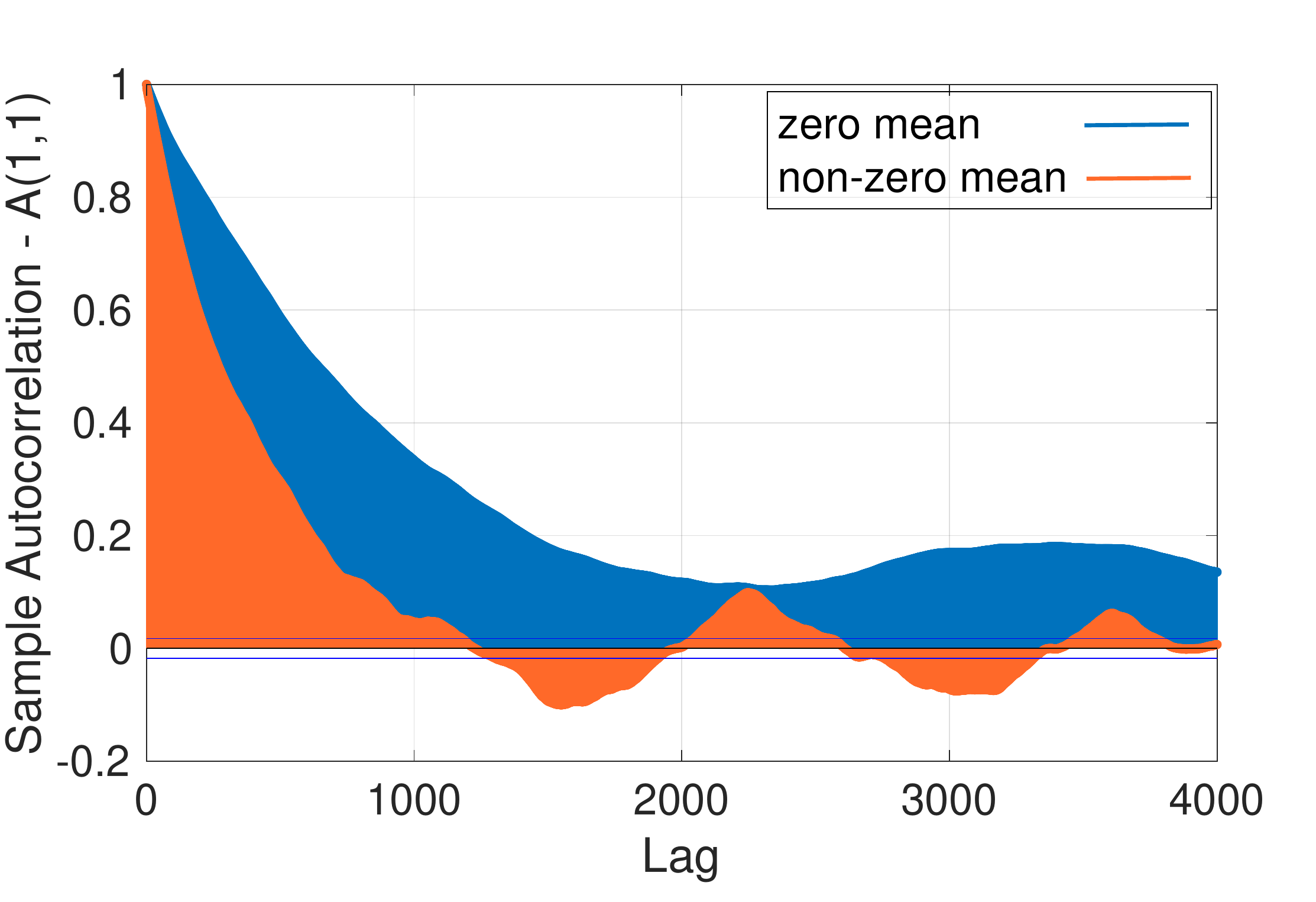}
    \caption{Gibbs}
    \label{fig:1}
  \end{subfigure}\hfil
  \begin{subfigure}{0.49\textwidth}
    \includegraphics[width=0.9\linewidth]{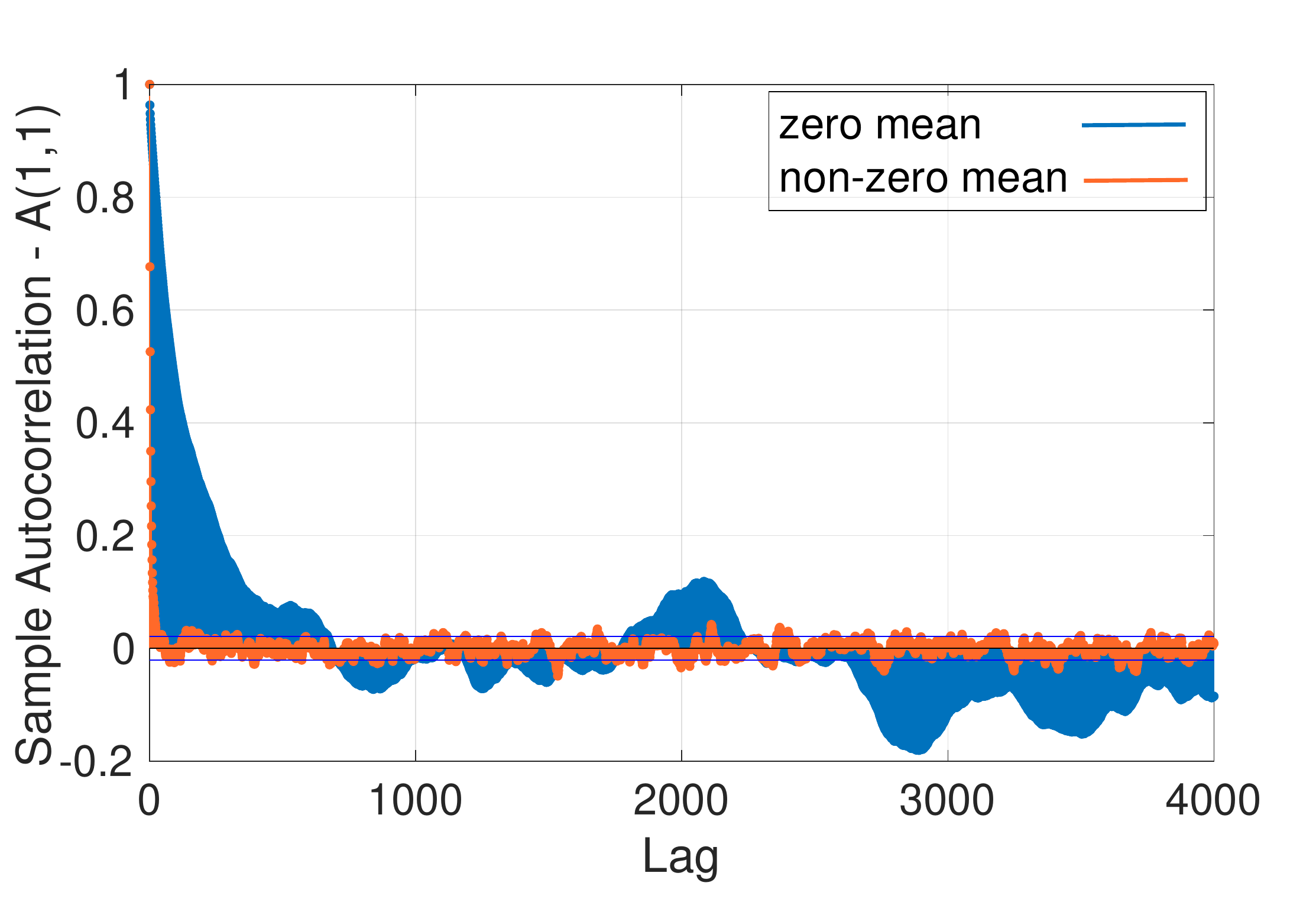}
    \caption{HMC}
    \label{fig:2}
  \end{subfigure}\hfil
  \caption{Autocorrelation for factor $a_{1,1}$ in Example~1 with zero and
  non-zero mean priors, computed using 10th chain of Gibbs and HMC samplers.}
  \label{fig:fig1}

  \includegraphics[width=0.45\textwidth]{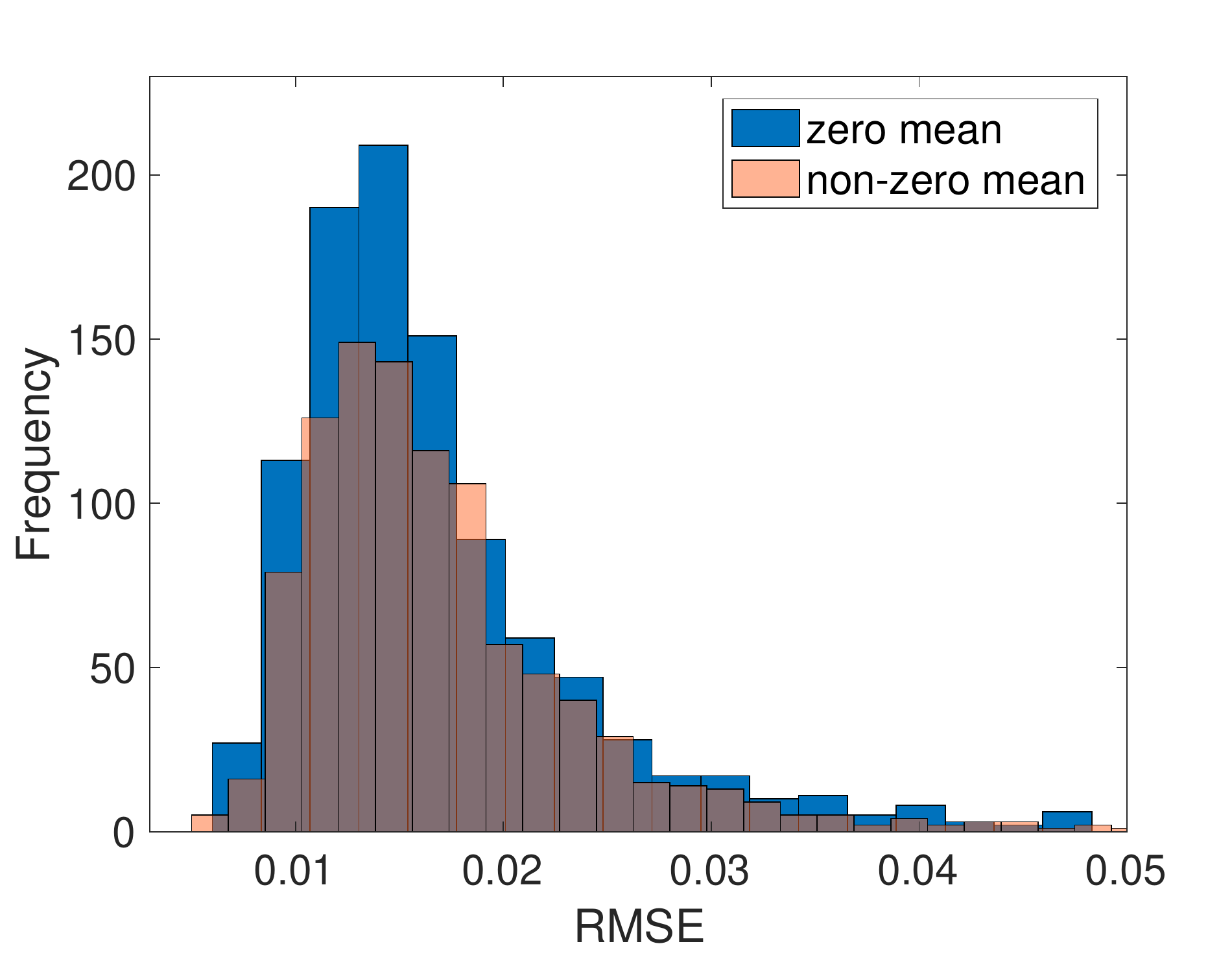}
  \caption{Reconstruction RMSE in Example~1 with zero and non-zero mean priors.}
  \label{fig:fig4}
\end{figure}

\subsection{Example 2: Partially Observed Synthetic Matrix}

As we pointed out earlier, we often do not have access to observations of all
entries of a matrix. To demonstrate the utility of our theory in context of
partial observations, we use the Gibbs sampler to reconstruct a rank-5 $100
\times 100$ matrix from observing only 20\% of its entries. Sampling is
performed with 5 chains, each with 5000 samples, and non-zero prior means are
once again sampled from the $\text{Uniform}(0, 1)$ distribution. The prior on the
precision is the same as the previous example.  We show the posterior
predictions for some elements in Figure~\ref{fig:fig5} corresponding to zero and
non-zero mean priors; results are presented for same elements in
Figure~\ref{fig:fig5}(a) and Figure~\ref{fig:fig5}(b). We can clearly see that
using non-zero mean priors results in better MCMC performance. The corresponding
order-of-magnitudes improvement in autocorrelation of the factor $b_{50,4}$ by
choosing non-zero mean priors is showcased in Figure~\ref{fig:fig6}. The RMSE
histogram constructed from 50 repetitions of the experiment is shown in
Figure~\ref{fig:fig7}---here we see that using non-zero mean priors leads to
better reconstruction errors.

\begin{figure}
  \centering
  \begin{subfigure}{0.48\textwidth}
    \includegraphics[trim=200 522 190 88, clip, width=0.9\linewidth]{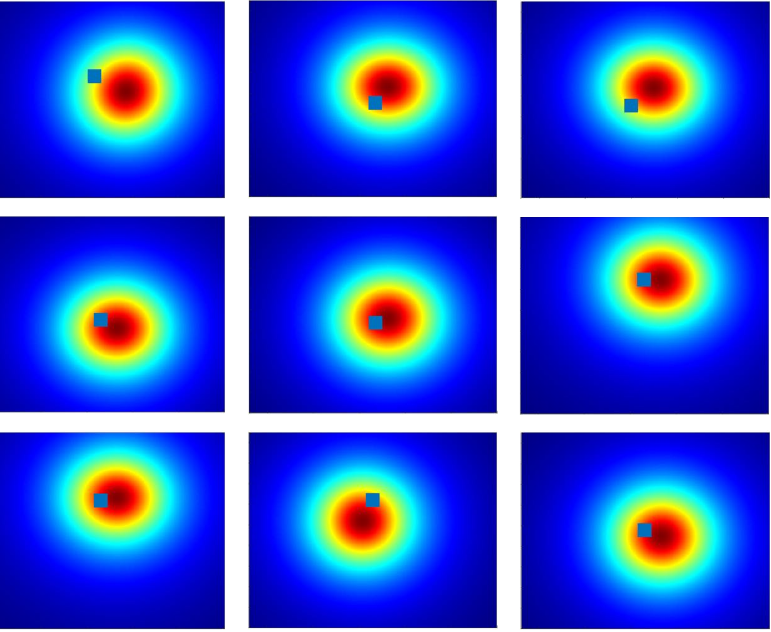}
    \caption{Zero mean priors}
    \label{fig:3}
  \end{subfigure}
  ~
  \begin{subfigure}{0.48\textwidth}
    \includegraphics[trim=200 535 190 75, clip, width=0.9\linewidth]{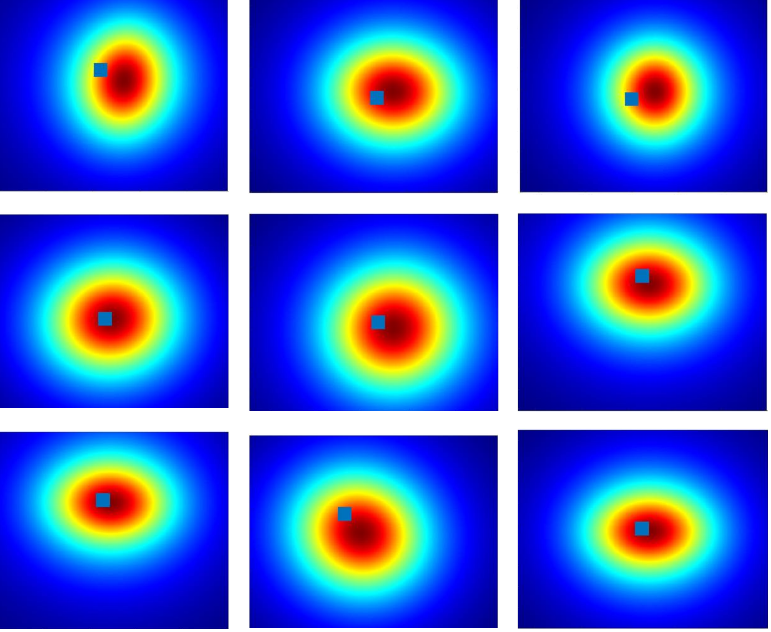}
    \caption{Non-zero mean priors}
    \label{fig:4}
  \end{subfigure}\hfil
  \caption{Posterior predictions for some elements in Example~2 with Gibbs
  sampler. Blue marks indicate truth.}
  \label{fig:fig5}

  \begin{minipage}{0.5\textwidth}
    \centering
    \includegraphics[width=0.95\linewidth]{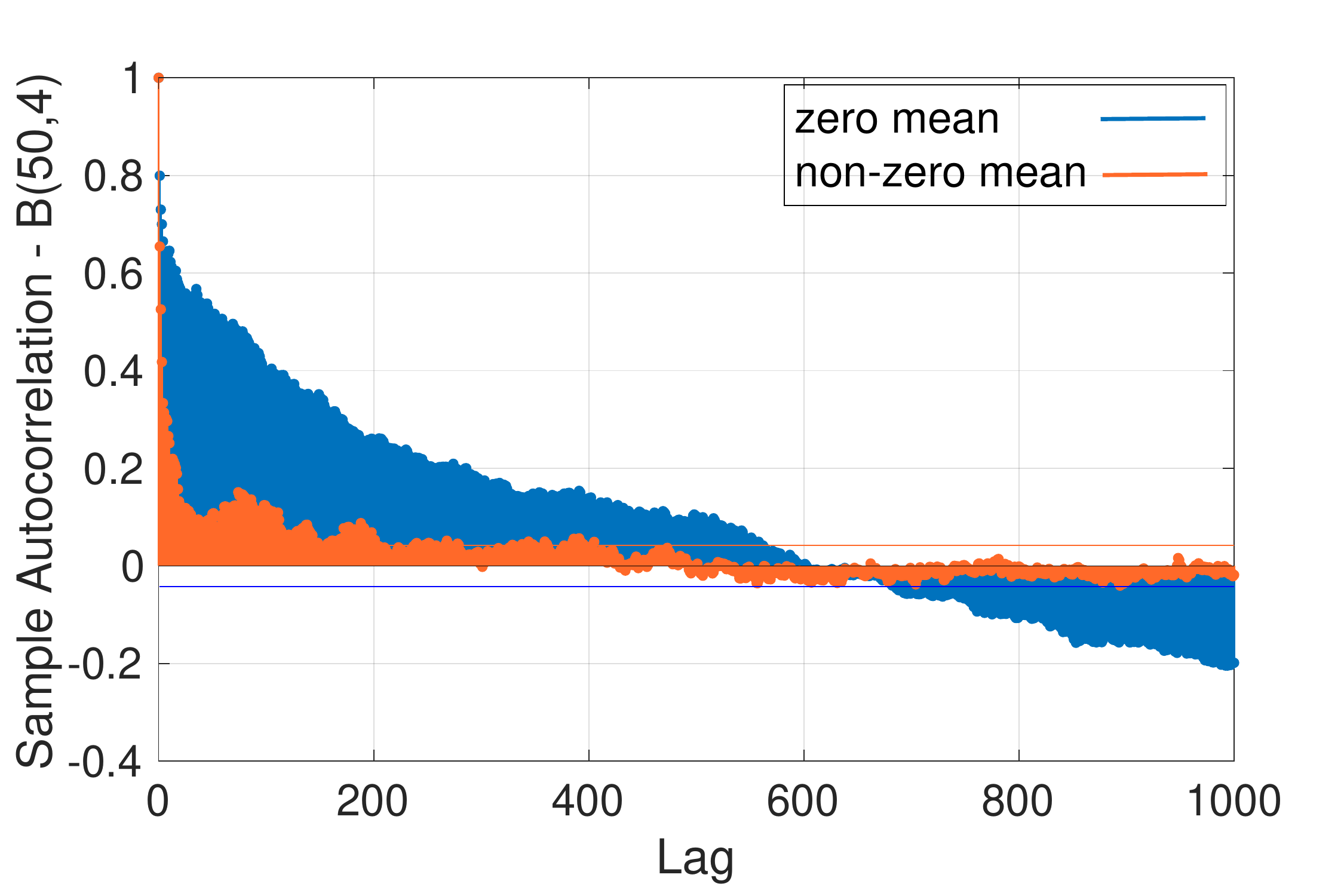}
    \caption{Autocorrelation for factor $b_{50,4}$ in Example~2 corresponding to
    zero and non-zero mean priors, computed with 5th chain of Gibbs sampler.}
    \label{fig:fig6}
  \end{minipage}
  ~
  \begin{minipage}{0.43\textwidth}
    \centering
    \includegraphics[width=0.9\linewidth]{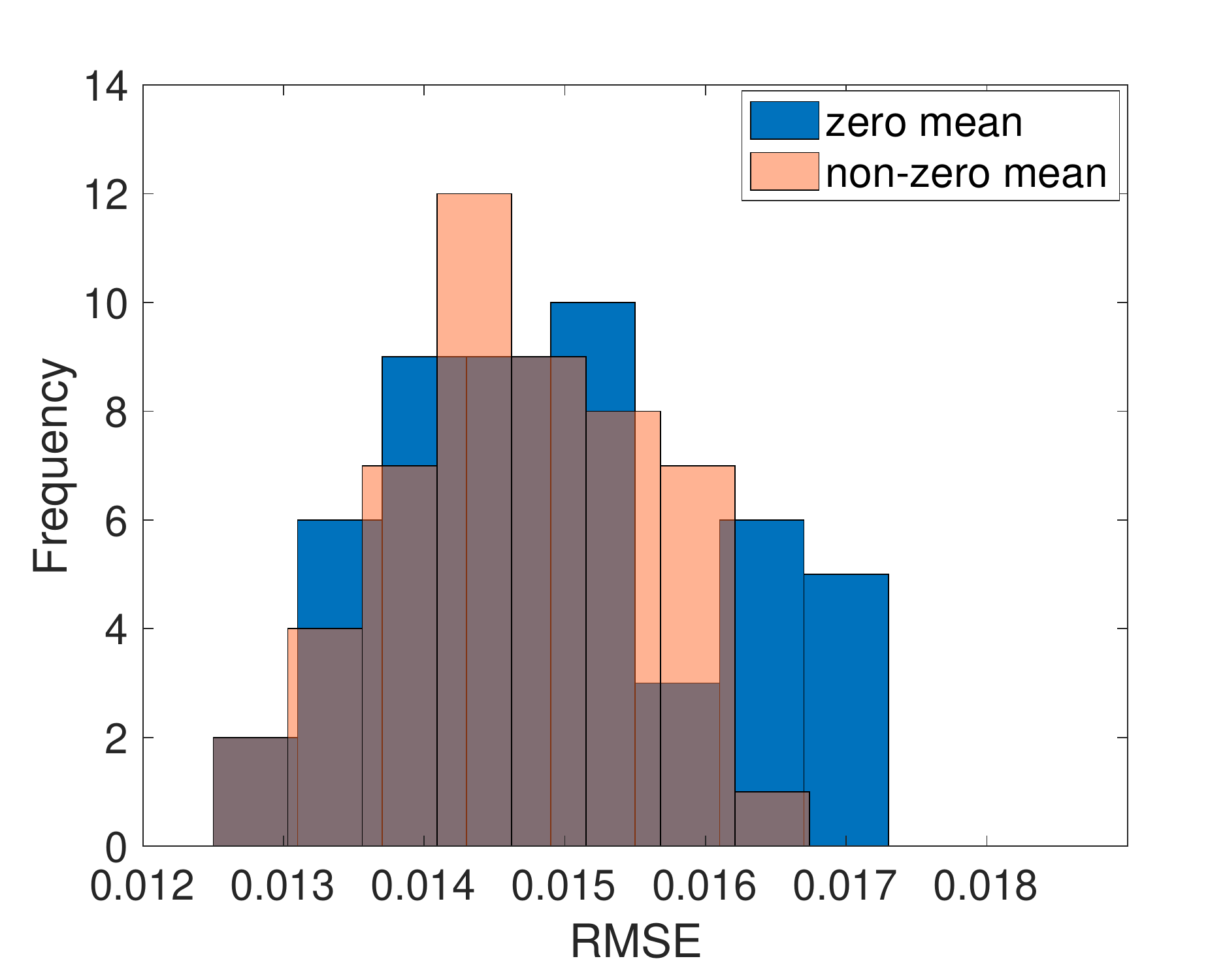}
    \caption{Reconstruction RMSE in Example~2 with zero and non-zero mean
    priors.}
    \label{fig:fig7}
  \end{minipage}
\end{figure}

\subsection{Example 3: Impaired Driving Dataset}

We now apply our theory to reconstruct the Impaired Driving Death Rate by Age
and Gender data set available publicly from CDC.%
\footnote{https://data.cdc.gov/Motor-Vehicle/Impaired-Driving-Death-Rate-by-Age-and-Gender-2012/ebbj-sh54}
We assume 60\% of the data is observed. Gibbs sampling is initiated with 5
chains, each with 10000 samples. Identical setup as in previous examples is
used for this experiment. We demonstrate that introducing non-zero mean priors
improves the autocorrelation for the $a_{20,1}$ and $a_{26,7}$ in
Figure~\ref{fig:fig8}. Again we see significant improvement in efficiency in
the same MCMC sampler.

\begin{figure}
  \centering
  \begin{subfigure}{0.49\textwidth}
    \centering
    \includegraphics[width=0.9\linewidth]{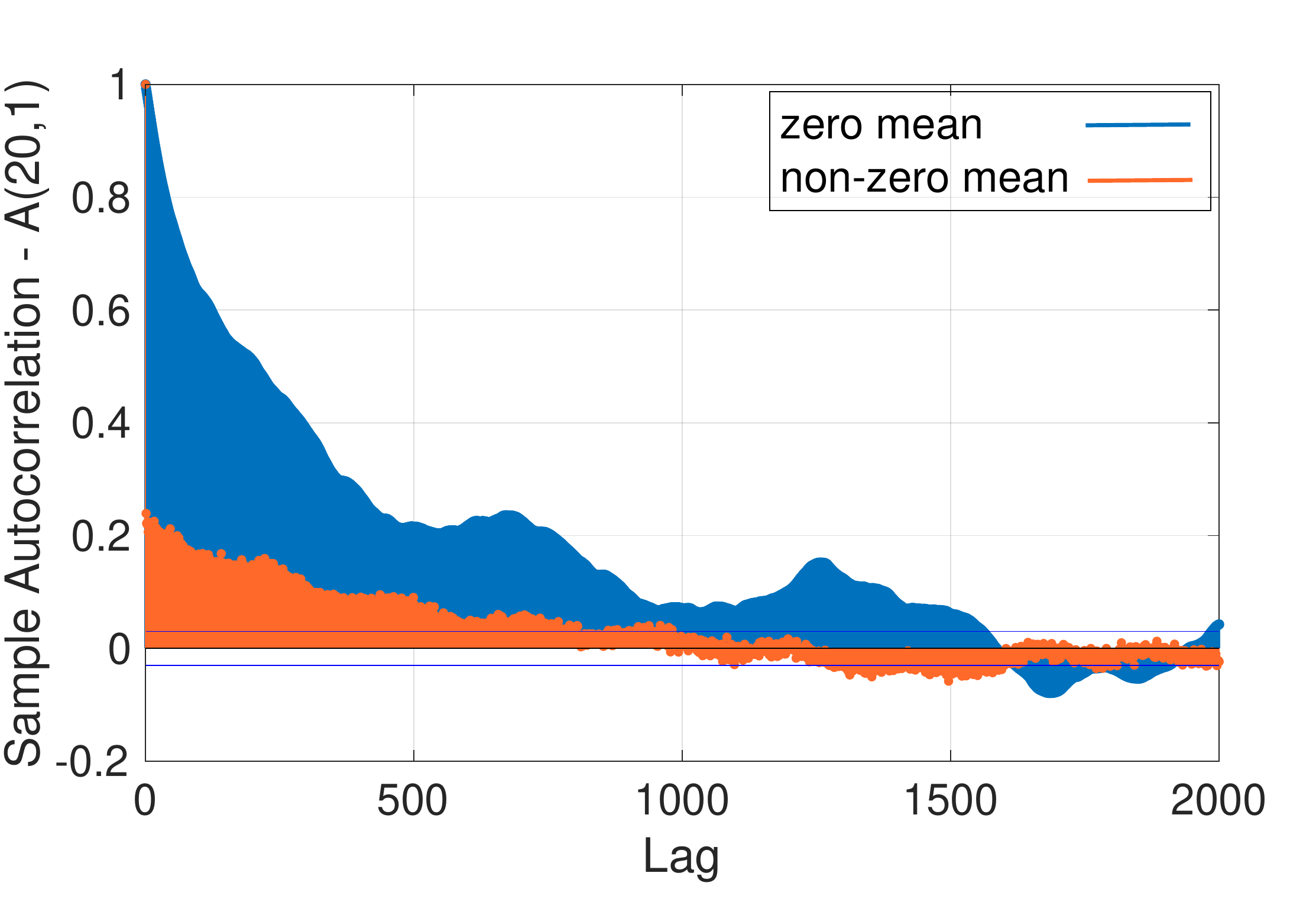}
    \caption{Factor $a_{20,1}$}
    \label{fig:5}
  \end{subfigure}\hfil
  \begin{subfigure}{0.49\textwidth}
    \centering
    \includegraphics[width=0.9\linewidth]{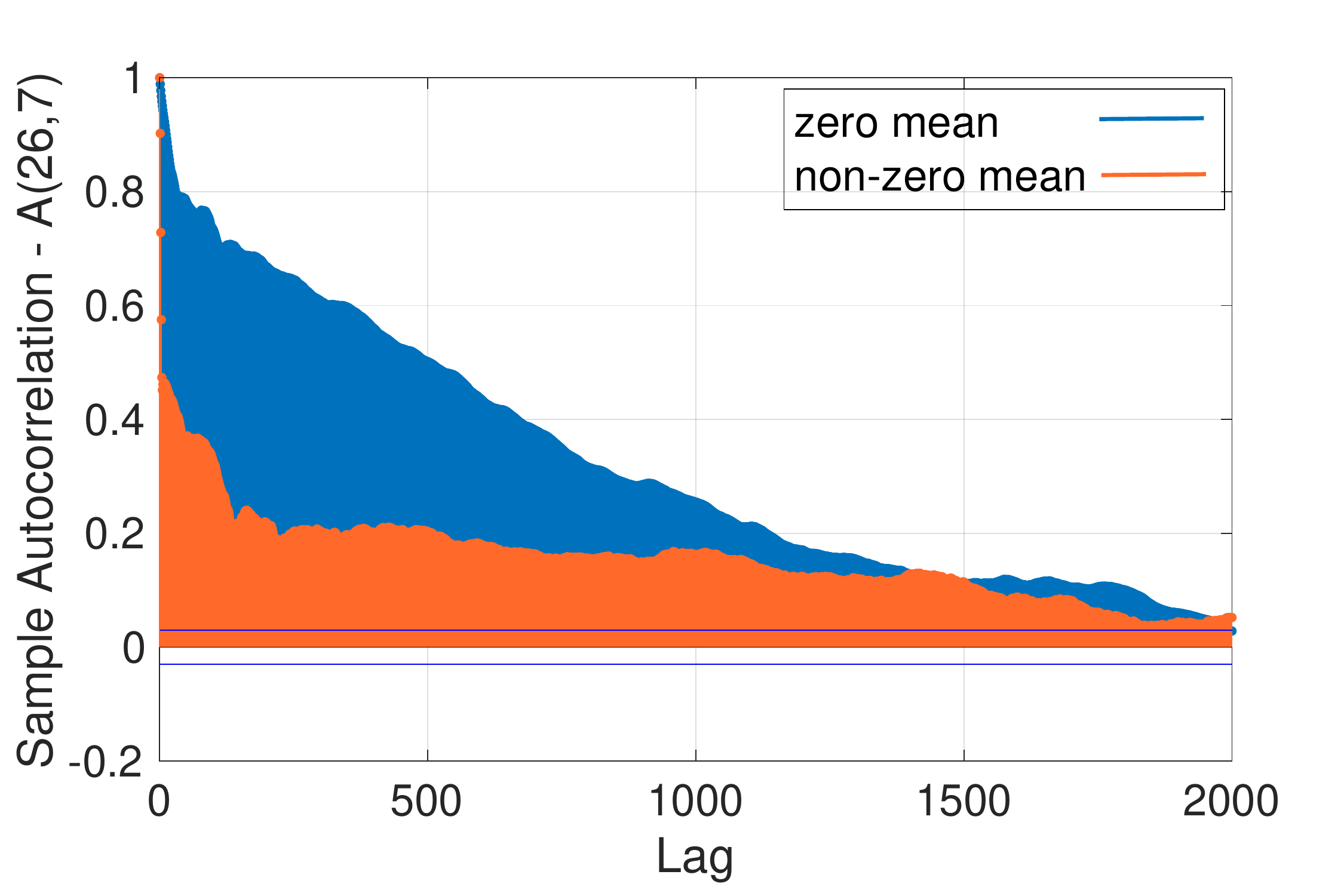}
    \caption{Factor $a_{26,7}$}
    \label{fig:6}
  \end{subfigure}\hfil
  \caption{Autocorrelation for some factors in Example~3 corresponding to zero
  and non-zero mean priors, computed from 3rd chain of Gibbs sampler.}
  \label{fig:fig8}
\end{figure}

\subsection{Example 4: Mice Protein Expression Dataset}

Finally, we apply our theory to the mice-protein expression data set
\citep{higuera2015self} available from the UCI Machine Learning Repository.%
\footnote{https://archive.ics.uci.edu/ml/datasets/Mice+Protein+Expression} The
data set consists of 77 protein expressions, measured in terms of nuclear
fractions, from 1080 mice specimens. For our experiments, we randomly
sub-sampled this $1080 \times 77$ matrix, creating a $50 \times 50$
fully-observed sub-matrix. We then constructed a rank 10 factorization of the
form \eqref{eq:factor-model} using Gibbs sampling while observing only 50\% of
the entries. Sampling was initiated for 4 chains, each with 20000 samples. We
set the parameter values
\begin{equation*}
  \tau_{a,1} = \cdots = \tau_{a,10} = 25, \quad \tau_{b,1} = \cdots =
  \tau_{b,10} = 25, \quad \tau_\eta = 10^2,
\end{equation*}
and the entries of the non-zero prior mean matrices were sampled from the
$\text{Uniform}(-\frac{7}{2}, \frac{7}{2})$ distribution. The first 1000 samples
from all chains were discarded as burn-in. We observe that using non-zero mean
priors leads to improvement of sample autocorrelation for the $b_{5,7}$ factor
in Figure~\ref{fig:mice-ac}. We also see improved reconstruction errors in
Figure~\ref{fig:mice-rmse}---the RMSE histograms were computed from 64
independent repetitions of the experiment described above, with ten-fold sample
thinning.

\begin{figure}
  \centering
  \begin{minipage}{0.49\textwidth}
    \centering
    \includegraphics[width=0.95\linewidth]{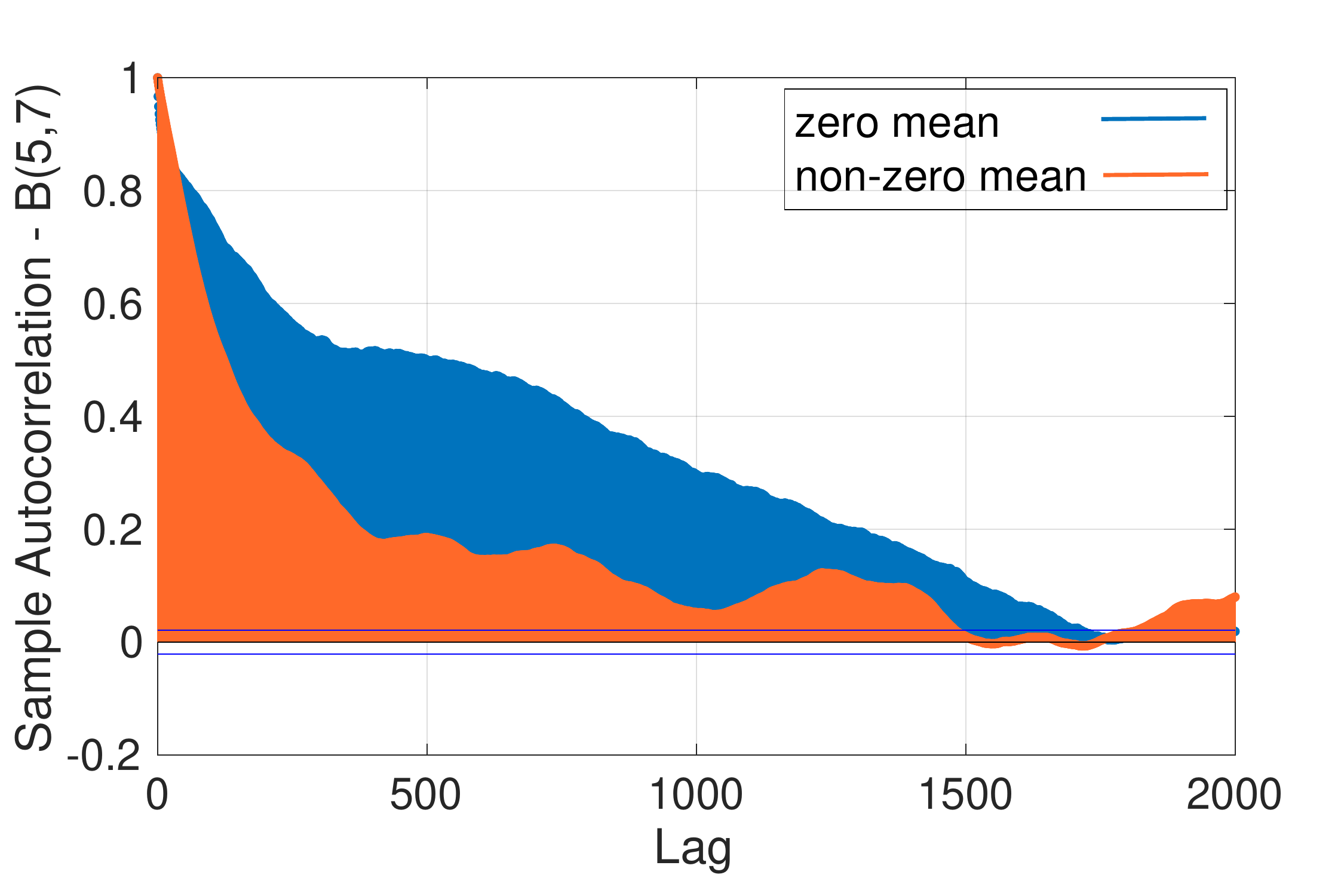}
    \caption{Autocorrelation for factor $b_{5,7}$ in Example~4 corresponding to
    zero/non-zero means, computed from 4th chain of Gibbs sampler.}
    \label{fig:mice-ac}
  \end{minipage}
  ~
  \begin{minipage}{0.43\textwidth}
    \centering
    \includegraphics[width=0.9\linewidth]{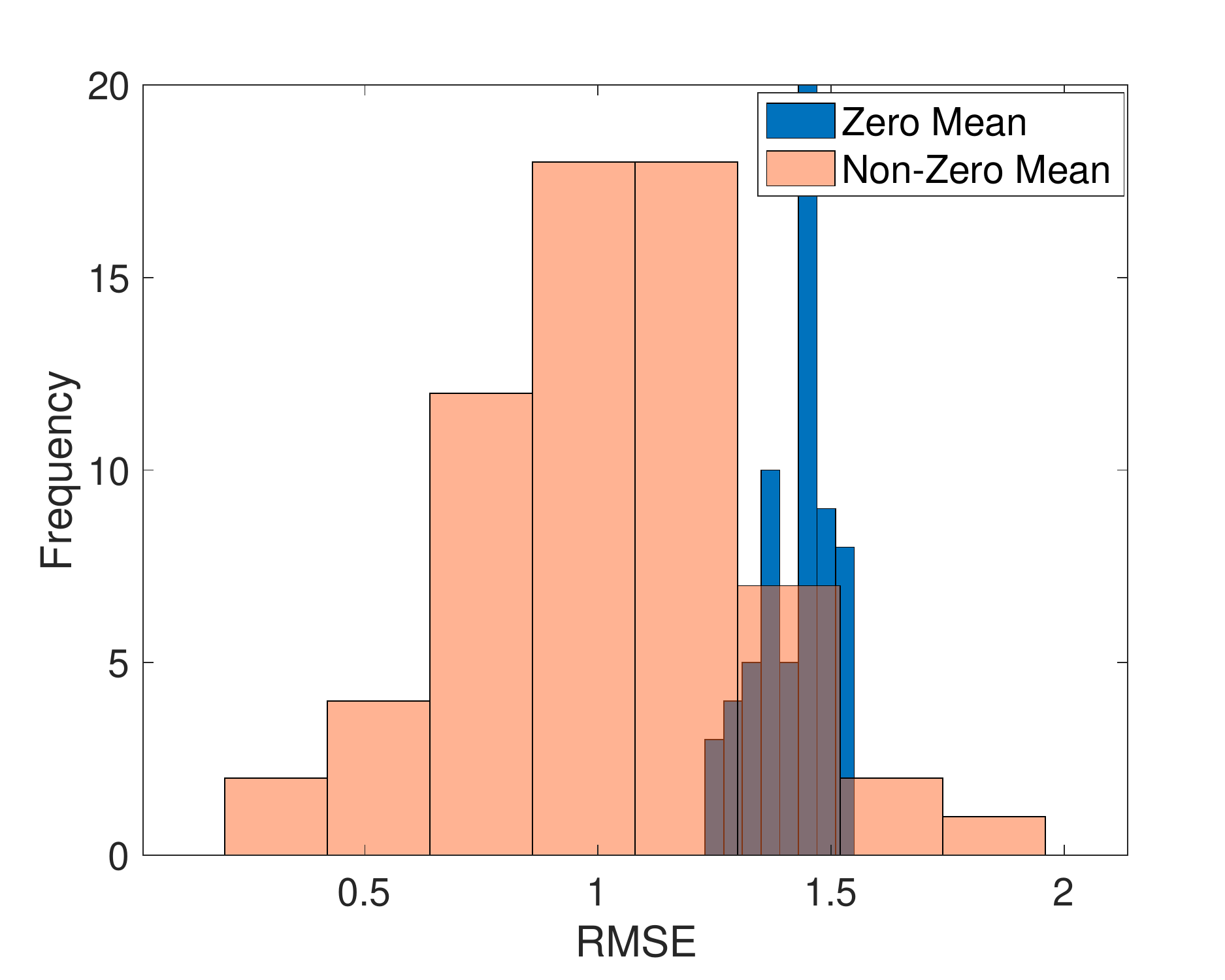}
    \caption{Reconstruction RMSE in Example~4 from 64 Gibbs sampling experiments
    with zero and non-zero mean priors.}
    \label{fig:mice-rmse}
  \end{minipage}
\end{figure}

\section{Conclusion}
\label{sec:conclusion}

We have presented a full theoretical treatment of the symmetries of posteriors
that arise from non-identifiability in Bayesian low-rank matrix factorization
due to the standard choice of Gaussian priors on the matrix factors.  We
established that using a carefully chosen set of prior means, we can eliminate
these symmetries, leading to better performance of MCMC sampling algorithms both
with synthetic and real-world data. In future, we intend to extend this
framework to address similar non-identifiability issues for low-rank tensor
factorizations.

\acks{%
  S.\ De acknowledges support from NSF under grant DMS-1454010 and from the
  Automotive Research Center (ARC) in accordance with Cooperative Agreement
  W56HZV-14-2-0001 with U.S. Army Ground Vehicle Systems Center.

  A.\ Gorodetsky and H.\ Salehi acknowledge support from the Department of
  Energy Office of Scientific Research, ASCR under grant DE-SC0020364.
}

\appendix

\section{Proof of Proposition~\ref{prop:individual-invariance}}
\label{app:individual-invariance}

\noindent
{\bf Proposition~\ref{prop:individual-invariance}}
{\it The posterior corresponding to
  \begin{equation*}
    \begin{split}
      -\log p(\Mat{A}, \Mat{B} \mid \Vec{y}) = \frac{1}{2} \sum_{k = 1}^r
      \tau_{a,k} \Norm{\Vec{a}_k - \Vec{\mu}_{a,k}}^2
      &+ \frac{1}{2} \sum_{k = 1}^r \tau_{b,k} \Norm{\Vec{b}_k -
      \Vec{\mu}_{b,k}}^2 \\
      &+ \frac{\tau_\eta}{2} \sum_{\Vec{\lambda} \in \Lambda} (y_{\Vec{\lambda}}
      - \RowVec{a}_{\lambda_1}^\top \RowVec{b}_{\lambda_2})^2 + \text{const.}
    \end{split}
  \end{equation*}
  is invariant under invertible transformation $\Mat{W} \in \RR^{r \times r}$,
  i.e.
  \begin{equation*}
    p(\Mat{A}, \Mat{B} \mid \Vec{y}) = p(\Mat{A} \Mat{W}, \Mat{B}
    \Mat{W}^{-\top} \mid \Vec{y}) \quad \text{for all} \quad \Mat{A} \in \RR^{m
    \times r} \text{ and } \Mat{B} \in \RR^{n \times r},
  \end{equation*}
  if and only if the terms
  \begin{align*}
    f_1(\Mat{A}) &= \sum_{k = 1}^r \tau_{a,k} \Norm{\Vec{a}_k}^2,
                 &
    f_2(\Mat{A}) &= \sum_{k = 1}^r \tau_{a,k} \Vec{\mu}_{a,k}^\top \Vec{a}_k, \\
    f_3(\Mat{B}) &= \sum_{k = 1}^r \tau_{b,k} \Norm{\Vec{b}_k}^2,
                 &
    f_4(\Mat{B}) &= \sum_{k = 1}^r \tau_{b,k} \Vec{\mu}_{b,k}^\top \Vec{b}_k
  \end{align*}
  are individually invariant under the $\Mat{A} \mapsto \Mat{A} \Mat{W}$ and
  $\Mat{B} \mapsto \Mat{B} \Mat{W}^{-\top}$ transformations.
}

\begin{proof}
  Note that
  \begin{equation*}
    \Norm{\Vec{a}_k - \Vec{\mu}_{a,k}}^2 = \Norm{\Vec{a}_k}^2 - 2
    \Vec{\mu}_{a,k}^\top \Vec{a}_k + \text{const.},
  \end{equation*}
  and similarly
  \begin{equation*}
    \Norm{\Vec{b}_k - \Vec{\mu}_{b,k}}^2 = \Norm{\Vec{b}_k}^2 - 2
    \Vec{\mu}_{b,k}^\top \Vec{b}_k + \text{const.}
  \end{equation*}
  Thus, we can rewrite the negative log posterior as
  \begin{equation}
    \label{eq:posterior-simple}
    -\log p(\Mat{A}, \Mat{B} \mid \Vec{y}) = \frac{1}{2} f_1(\Mat{A}) -
    f_2(\Mat{A}) + \frac{1}{2} f_3(\Mat{B}) - f_4(\Mat{B}) - \log p(\Vec{y} \mid
    \Mat{A}, \Mat{B}) + \text{const.}
  \end{equation}
  From Proposition~\ref{prop:likelihood-invariance}, we see that the
  likelihood term is already invariant under the invertible transformation. It
  follows invariance of $f_1$, $f_2$, $f_3$ and $f_4$ is sufficient for
  invariance of the posterior.

  We now establish that invariance of $f_1$, $f_2$, $f_3$ and $f_4$ is necessary
  for invariance of the posterior. To show this, first note that $f_1$, $f_3$
  are homogeneous of degree two, and $f_2$, $f_4$ are homogeneous of degree one.
  More explicitly, for all $t, s \in \RR$ we have
  \begin{align*}
    f_1(t \Mat{A}) &= t^2 f_1(\Mat{A}),
                   &
    f_2(t \Mat{A}) &= t f_2(\Mat{A}), \\
    f_3(s \Mat{B}) &= s^2 f_3(\Mat{B}),
                   &
    f_4(s \Mat{B}) &= s f_4(\Mat{B}).
  \end{align*}
  Now, fix $\Mat{A}$ and $\Mat{B}$, then for invariance
  we must have $p(t \Mat{A}, s \Mat{B} \mid \Vec{y}) = p(t \Mat{A} \Mat{W}, s
  \Mat{B} \Mat{W}^{-\top} \mid \Vec{y})$ for all $s, t \in \RR$. Expanding this
  out using \eqref{eq:posterior-simple} and applying the homogeneity
  properties, we obtain
  \begin{equation*}
    \frac{t^2}{2} f_1(\Mat{A}) - t f_2(\Mat{A}) + \frac{s^2}{2} f_3(\Mat{B}) - s
    f_4(\Mat{B}) = \frac{t^2}{2} f_1(\Mat{A} \Mat{W}) - t f_2(\Mat{A} \Mat{W}) +
    \frac{s^2}{2} f_3(\Mat{B} \Mat{W}^{-\top}) - s f_4(\Mat{B} \Mat{W}^{-\top}),
  \end{equation*}
  where the likelihood term cancels out. Comparing the coefficients of
  like-powered terms on the both sides, we conclude that we must have
  \begin{align*}
    f_1(\Mat{A}) &= f_1(\Mat{A} \Mat{W})
                 &
    f_2(\Mat{A}) &= f_2(\Mat{A} \Mat{W}) \\
    f_3(\Mat{B}) &= f_3(\Mat{B} \Mat{W}^{-\top})
                 &
    f_4(\Mat{B}) &= f_4(\Mat{B} \Mat{W}^{-\top})
  \end{align*}
  Since $\Mat{A}$ and $\Mat{B}$ are arbitrary, it follows that $f_1$, $f_2$,
  $f_3$ and $f_4$ must be individually invariant under the $\Mat{A} \mapsto
  \Mat{A} \Mat{W}$ and $\Mat{B} \mapsto \Mat{B} \Mat{W}^{-\top}$
  transformations.
\end{proof}

\section{Proof of Theorem~\ref{thm:posterior-symmetry}}
\label{app:posterior-symmetry}

We will use the following result:
\begin{lemma}
  \label{lem:orth}
  Let $\Mat{Q} \in \RR^{r \times r}$ be a matrix satisfying $\Norm{\Mat{Q}
  \Vec{x}} = \Norm{\Vec{x}}$ for all $\Vec{x} \in \RR^r$. Then $\Mat{Q}$ is
  orthogonal.
\end{lemma}

\begin{proof}
  Let $\Vec{x}_1$ and $\Vec{x}_2$ be two arbitrary vectors. Then
  \begin{equation*}
    \Norm{\Mat{Q} \Vec{x}_1 + \Mat{Q} \Vec{x}_2}^2 = \Norm{\Mat{Q} \Vec{x}_1}^2
    + \Norm{\Mat{Q} \Vec{x}_2}^2 + 2 \InProd{\Mat{Q} \Vec{x}_1}{\Mat{Q}
    \Vec{x}_2}
  \end{equation*}
  and
  \begin{equation*}
    \Norm{\Vec{x}_1 + \Vec{x}_2}^2 = \Norm{\Vec{x}_1}^2 + \Norm{\Vec{x}_2}^2 + 2
    \InProd{\Vec{x}_1}{\Vec{x}_2}
  \end{equation*}
  Using the facts $\Norm{\Mat{Q} \Vec{x}_1 + \Mat{Q} \Vec{x}_2} = \Norm{\Mat{Q}
  (\Vec{x}_1 + \Vec{x}_2)} = \Norm{\Vec{x}_1 + \Vec{x}_2}$, $\Norm{\Mat{Q}
  \Vec{x}_1} = \Norm{\Vec{x}_1}$ and $\Norm{\Mat{Q} \Vec{x}_2} =
  \Norm{\Vec{x}_2}$ we obtain
  \begin{equation*}
    \InProd{\Mat{Q} \Vec{x}_1}{\Mat{Q} \Vec{x}_2} =
    \InProd{\Vec{x}_1}{\Vec{x}_2} \quad \text{for all} \quad \Vec{x}_1,
    \Vec{x}_2 \in \RR^r
  \end{equation*}
  Setting $\Vec{x}_1 = \Vec{e}_i$ and $\Vec{x}_2 = \Vec{e}_j$, we obtain
  \begin{equation*}
    \InProd{\Vec{q}_i}{\Vec{q}_j} = \InProd{\Vec{e}_i}{\Vec{e}_j}
  \end{equation*}
  i.e.\ $\{\Vec{q}_1, \ldots, \Vec{q}_r\}$ form an orthonormal system in
  $\RR^r$, hence $\Mat{Q}$ is orthogonal.
\end{proof}

For convenience, we restate Theorem~\ref{thm:posterior-symmetry} slightly
differently than earlier:

\noindent
{\bf Theorem~\ref{thm:posterior-symmetry}}
{\it
  Let us denote
  \begin{equation*}
    \Mat{T}_a = \Diag(\tau_{a, 1}, \ldots, \tau_{a, r}), \quad \Mat{T}_b =
    \Diag(\tau_{b, 1}, \ldots, \tau_{b, r})
  \end{equation*}
  Let $\{\Lambda_1, \ldots, \Lambda_q\}$ be a partition of $\{1, \ldots, r\}$
  defined by the equivalence relation
  \begin{equation*}
    k, k' \in \Lambda_\ell \iff \tau_{a, k} \tau_{b, k} = \tau_{a, k'} \tau_{b,
    k'}
  \end{equation*}
  Then the posterior given by
  \begin{equation}
    \label{eq:posterior-symmetry-zero}
    -\log p(\Mat{A}, \Mat{B} \mid \Vec{y}) = \frac{1}{2} \sum_{k = 1}^r
    \tau_{a,k} \Norm{\Vec{a}_k}^2 + \frac{1}{2} \sum_{k = 1}^r \tau_{b,k}
    \Norm{\Vec{b}_k}^2 + \frac{\tau_\eta}{2} \sum_{\Vec{\lambda} \in \Lambda}
    (y_{\Vec{\lambda}} - \RowVec{a}_{\lambda_1}^\top \RowVec{b}_{\lambda_2})^2 +
    \text{const.}
  \end{equation}
  is invariant under transformation $(\Mat{A}, \Mat{B}) \mapsto (\Mat{A}
  \Mat{W}, \Mat{B} \Mat{W}^{-\top})$ with invertible $\Mat{W} \in \RR^{r \times
  r}$ if and only if we can decompose
  \begin{equation}
    \label{eq:posterior-symmetry-cond-1}
    \Mat{W} = \Mat{T}_a^{1/2} \Mat{Q} \Mat{T}_a^{-1/2} = \Mat{T}_b^{-1/2}
    \Mat{Q} \Mat{T}_b^{1/2}
  \end{equation}
  where $\Mat{Q}$ is orthogonal and block diagonal w.r.t.\ the partition
  $\{\Lambda_1, \ldots, \Lambda_q\}$, i.e. the sub-matrices
  \begin{equation}
    \label{eq:posterior-symmetry-cond-2}
    \Mat{Q}[\Lambda_{\ell_1}, \Lambda_{\ell_2}] \text{ are }
    \begin{cases}
      \text{orthogonal} & \text{ if } \ell_1 = \ell_2 \\
      \text{zero} & \text{ if } \ell_1 \neq \ell_2
    \end{cases}
  \end{equation}
}

\begin{proof}
  From Proposition~\ref{prop:individual-invariance}, invertible invariance
  \eqref{eq:posterior-symmetry-zero} holds if and only if the terms
  \begin{equation*}
    f_1(\Mat{A}) = \sum_{k = 1}^r \tau_{a, k} \Norm{\Vec{a}_k}^2, \quad
    f_3(\Mat{B}) = \sum_{k = 1}^r \tau_{b, k} \Norm{\Vec{b}_k}^2
  \end{equation*}
  are invariant under the $\Mat{A} \mapsto \Mat{A} \Mat{W}$ and $\Mat{B} \mapsto
  \Mat{B} \Mat{W}^{-\top}$ transformations.

  We divide the rest of the proof into three steps:
  \begin{itemize}
    \item
      Step 1 derives the following condition on $\Mat{W}$ necessary for
      invertible invariance:
      \begin{equation}
        \label{eq:posterior-symmetry-cond-1-alt}
        \Mat{Q}_a = \Mat{T}_a^{-1/2} \Mat{W} \Mat{T}_a^{1/2} \quad \text{and}
        \quad \Mat{Q}_b = \Mat{T}_b^{1/2} \Mat{W} \Mat{T}_b^{-1/2}
      \end{equation}
      must be orthogonal matrices.
    \item
      Step 2 establishes $\Mat{Q}_a = \Mat{Q}_b$, and derives the block diagonal
      structure \eqref{eq:posterior-symmetry-cond-2} on this common
      orthogonal matrix $\Mat{Q}$. This establishes
      \eqref{eq:posterior-symmetry-cond-1} as a necessary condition for
      invertible invariance.
    \item
      Step 3 proves that \eqref{eq:posterior-symmetry-cond-1} is in fact
      sufficient for invertible invariance.
  \end{itemize}

  \noindent{\textbf{Step 1.}} Let us choose
  \begin{equation*}
    \Mat{A} =
    \begin{bmatrix}
      \alpha_1 & \cdots & \alpha_r \\
      0 & \cdots & 0 \\
      \vdots & \ddots & \vdots \\
      0 & \cdots & 0
    \end{bmatrix}_{m \times r}
    \implies
    \Mat{A} \Mat{W} =
    \begin{bmatrix}
      \Vec{\alpha}^\top \\
      \Vec{0}^\top \\
      \vdots \\
      \Vec{0}^\top
    \end{bmatrix}
    \begin{bmatrix}
      \Vec{w}_1 & \cdots & \Vec{w}_r
    \end{bmatrix}
    =
    \begin{bmatrix}
      \Vec{\alpha}^\top \Vec{w}_1 & \cdots & \Vec{\alpha}^\top \Vec{w}_r \\
      0 & \cdots & 0 \\
      \vdots & \ddots & \vdots \\
      0 & \cdots & 0
    \end{bmatrix}_{m \times r}
  \end{equation*}
  for arbitrary $\Vec{\alpha} \in \RR^r$. Then we have
  \begin{align*}
    f(\Mat{A}) &= \sum_{k = 1}^r \tau_{a, k} \alpha_k^2 = \Norm{\Mat{T}_a^{1/2}
    \Vec{\alpha}}^2 \\
    f(\Mat{A} \Mat{W}) &= \sum_{k = 1}^r \tau_{a, k} (\Vec{\alpha}^\top
    \Vec{w}_k)^2 = \sum_{k = 1}^r \tau_{a, k} (\Vec{w}_k^\top \Vec{\alpha})^2 =
    \Norm{\Mat{T}_a^{1/2} \Mat{W}^\top \Vec{\alpha}}^2
  \end{align*}
  These equalities follow from the following observations:
  \begin{equation*}
    \Mat{T}_a^{1/2} \Vec{\alpha} =
    \begin{bmatrix}
      \tau_{a, 1}^{1/2} \alpha_1 \\
      \vdots \\
      \tau_{a, r}^{1/2} \alpha_r
    \end{bmatrix},
    \quad  \Mat{T}_a^{1/2} \Mat{W}^\top \Vec{\alpha} = \Mat{T}_a^{1/2}
    \begin{bmatrix}
      \Vec{w}_1^\top \Vec{\alpha} \\
      \vdots \\
      \Vec{w}_r^\top \Vec{\alpha}
    \end{bmatrix}
    =
    \begin{bmatrix}
      \tau_{a, 1}^{1/2} \Vec{w}_1^\top \Vec{\alpha} \\
      \vdots \\
      \tau_{a, r}^{1/2} \Vec{w}_r^\top \Vec{\alpha}
    \end{bmatrix}
  \end{equation*}
  Thus, invariance of $f_1$ under $\Mat{A} \mapsto \Mat{A} \Mat{W}$
  transformation requires
  \begin{equation*}
    \Norm{\Mat{T}_a^{1/2} \Vec{\alpha}} = \Norm{\Mat{T}_a^{1/2} \Mat{W}^\top
    \Vec{\alpha}} \quad \text{for all} \quad \Vec{\alpha} \in \RR^r
  \end{equation*}
  Set $\Vec{\beta} = \Mat{T}_a^{1/2} \Vec{\alpha}$, then by the invertibility of
  $\Mat{T}_a$ we can rewrite this condition as
  \begin{equation*}
    \Norm{\Vec{\beta}} = \Norm{\underbrace{\Mat{T}_a^{1/2} \Mat{W}^\top
      \Mat{T}_a^{-1/2}}_{\Mat{Q}_a^\top}
    \Vec{\beta}} \quad \text{for all} \quad \Vec{\beta} \in \RR^r
  \end{equation*}
  Lemma~\ref{lem:orth} then implies that $\Mat{Q}_a^\top$, and consequently
  $\Mat{Q}_a$ (as defined in \eqref{eq:posterior-symmetry-cond-1-alt}),
  must be orthogonal. Proceeding in a similar manner, we can show that the
  invariance of $f_3$ under the $\Mat{B} \mapsto \Mat{B} \Mat{W}^{-\top}$
  transformation would require $\Mat{Q}_b$ (again, as defined in
  \eqref{eq:posterior-symmetry-cond-1-alt}) to be orthogonal.

  \noindent{\textbf{Step 2.}} Using
  \eqref{eq:posterior-symmetry-cond-1-alt} we can compute
  \begin{equation*}
    \Mat{Q}_a^{-1} = \Mat{T}_a^{-1/2} \Mat{W}^{-1} \Mat{T}_a^{1/2} \implies
    \Mat{Q}_a^{-\top} = \Mat{T}_a^{1/2} \Mat{W}^{-\top} \Mat{T}_a^{-1/2}
  \end{equation*}
  But since $\Mat{Q}_a$ is orthogonal we have
  \begin{equation}
    \label{eq:posterior-symmetry-eq-1}
    \Mat{Q}_a =\Mat{Q}_a^{-\top} \implies \Mat{T}_a^{-1/2} \Mat{W}
    \Mat{T}_a^{1/2} = \Mat{T}_a^{1/2} \Mat{W}^{-\top} \Mat{T}_a^{-1/2} \implies
    \Mat{W} \Mat{T}_a = \Mat{T}_a \Mat{W}^{-\top}
  \end{equation}
  Similarly, from \eqref{eq:posterior-symmetry-cond-1-alt} and
  orthogonality of $\Mat{Q}_b$, we can derive
  \begin{equation}
    \label{eq:posterior-symmetry-eq-2}
    \Mat{T}_b \Mat{W} = \Mat{W}^{-\top} \Mat{T}_b
  \end{equation}
  Using \eqref{eq:posterior-symmetry-eq-1} and
  \eqref{eq:posterior-symmetry-eq-2} along with associativity of matrix
  multiplication, we obtain
  \begin{equation*}
    \Mat{T}_a \Mat{T}_b \Mat{W} = \Mat{T}_a (\Mat{T}_b \Mat{W})
    \overset{\eqref{eq:posterior-symmetry-eq-2}}{=} \Mat{T}_a
    (\Mat{W}^{-\top} \Mat{T}_b) = (\Mat{T}_a \Mat{W}^{-\top}) \Mat{T}_b
    \overset{\eqref{eq:posterior-symmetry-eq-1}}{=} (\Mat{W}
    \Mat{T}_a) \Mat{T_b} = \Mat{W} \Mat{T}_a \Mat{T}_b
  \end{equation*}
  Now, equating $(i, j)$-th entries of the two boundary matrices in the above
  chain (which are easy to compute given diagonal $\Mat{T}_a$ and $\Mat{T}_b$),
  we get
  \begin{equation*}
    \tau_{a, i} \tau_{b, i} w_{ij} = w_{ij} \tau_{a, j} \tau_{b, j} \quad
    \text{for all} \quad 1 \leq i, j \leq r
  \end{equation*}
  Clearly, if $\tau_{a, i} \tau_{b, i} \neq \tau_{a, j} \tau_{b, j}$ for some
  pair of indices $(i, j)$, then we must have $w_{ij} = 0$. This leads us to the
  block-diagonal structure of $\Mat{W}$ w.r.t.\ partition $\{\Lambda_1, \ldots,
  \Lambda_q\}$, i.e.\ $\Mat{W}[\Lambda_{\ell_1}, \Lambda_{\ell_2}]$ is non-zero
  only if $\ell_1 = \ell_2$. Using this with the diagonal nature of $\Mat{T}_a$
  and $\Mat{T}_b$ in \eqref{eq:posterior-symmetry-cond-1-alt}, we can
  conclude $\Mat{Q}_a$ and $\Mat{Q}_b$ have the same block-diagonal structure.

  Next, for each $1 \leq \ell \leq q$ we have $\tau_{a, i} \tau_{b, i} =
  \tau_{a, j} \tau_{b, j}$ for all $i, j \in \Lambda_\ell$. Let us call this
  common value $c_\ell$, then we have
  \begin{equation*}
    \begin{split}
      \Mat{T}_a[\Lambda_\ell, \Lambda_\ell] \Mat{T}_b[\Lambda_\ell, \Lambda_\ell]
      &= \Diag(\tau_{a, i} : i \in \Lambda_\ell) \Diag(\tau_{b, i} : i \in
      \Lambda_\ell) \\
      &= \Diag(\tau_{a, i} \tau_{b, i} : i \in \Lambda_\ell) \\
      &= \Diag(c_\ell : i \in \Lambda_\ell) \\
      &= c_\ell \Mat{I}_{r_\ell}
    \end{split}
  \end{equation*}
  where we denote $r_\ell = \Card{\Lambda_\ell}$. We conclude
  \begin{equation*}
    \begin{split}
      \Mat{T}_b[\Lambda_\ell, \Lambda_\ell]
      &= c_\ell \Mat{T}_a[\Lambda_\ell, \Lambda_\ell]^{-1} \\
      \implies \Mat{Q}_b[\Lambda_\ell, \Lambda_\ell]
      &= \Mat{T}_b[\Lambda_\ell, \Lambda_\ell]^{1/2} \Mat{W}[\Lambda_\ell,
      \Lambda_\ell] \Mat{T}_b[\Lambda_\ell, \Lambda_\ell]^{-1/2} \\
      &= c_\ell^{1/2} \Mat{T}_a[\Lambda_\ell, \Lambda_\ell]^{-1/2}
      \Mat{W}[\Lambda_\ell, \Lambda_\ell] c_\ell^{-1/2} \Mat{T}_a[\Lambda_\ell,
      \Lambda_\ell]^{1/2} \\
      &= \Mat{T}_a[\Lambda_\ell, \Lambda_\ell]^{-1/2} \Mat{W}[\Lambda_\ell,
      \Lambda_\ell] \Mat{T}_a[\Lambda_\ell, \Lambda_\ell]^{1/2} \\
      &= \Mat{Q}_a[\Lambda_\ell, \Lambda_\ell]
    \end{split}
  \end{equation*}
  Combining these for all the blocks, we obtain $\Mat{Q}_a = \Mat{Q}_b$. We call
  this common value $\Mat{Q}$, and \eqref{eq:posterior-symmetry-cond-1}
  is trivially satisfied.

  \noindent{\textbf{Step 3.}} This part of the
  proof is taken from \citet[Appendix G.3]{nakajima2011theoretical}:

  Note that we can write
  \begin{equation*}
    f_1(\Mat{A}) = \sum_{k = 1}^r \tau_{a, k} \Norm{\Vec{a}_k}^2 =
    \Trace((\Mat{T}_a \Mat{A}^\top) \Mat{A}) = \Trace(\Mat{A} (\Mat{T}_a
    \Mat{A}^\top)) = \Trace(\Mat{A} \Mat{T}_a \Mat{A}^\top)
  \end{equation*}
  where the second equality follows from the following observation:
  \begin{equation*}
    \begin{split}
      (\Mat{T}_a \Mat{A}^\top) \Mat{A}
      &=
      \left(
      \begin{bmatrix}
        \tau_{a, 1} & & \\
                    & \ddots & \\
                    & & \tau_{a, r}
      \end{bmatrix}
      \begin{bmatrix}
        \Vec{a}_1^\top \\
        \vdots \\
        \Vec{a}_r^\top
      \end{bmatrix}
      \right)
      \begin{bmatrix}
        \Vec{a}_1 & \cdots & \Vec{a}_r
      \end{bmatrix} \\
      &=
      \begin{bmatrix}
        \tau_{a, 1} \Vec{a}_1^\top \\
        \vdots \\
        \tau_{a, r} \Vec{a}_r^\top
      \end{bmatrix}
      \begin{bmatrix}
        \Vec{a}_1 & \cdots & \Vec{a}_r
      \end{bmatrix} \\
      &=
      \begin{bmatrix}
        \tau_{a, 1} \Vec{a}_1^\top \Vec{a}_1 & \cdots & \tau_{a, 1}
        \Vec{a}_1^\top \Vec{a}_r \\
        \vdots & \ddots & \vdots \\
        \tau_{a, 1} \Vec{a}_r^\top \Vec{a}_1 & \cdots & \tau_{a, 1}
        \Vec{a}_r^\top \Vec{a}_r
      \end{bmatrix}
    \end{split}
  \end{equation*}
  Now, using $\Mat{W} = \Mat{T}_a^{1/2} \Mat{Q} \Mat{T}_a^{-1/2}$ from
  \eqref{eq:posterior-symmetry-cond-1}, we get
  \begin{equation*}
    \begin{split}
      f_1(\Mat{A} \Mat{W})
      &= \Trace((\Mat{A} \Mat{W}) \Mat{T}_a (\Mat{A} \Mat{W})^\top) \\
      &= \Trace(\Mat{A} \Mat{W} \Mat{T}_a \Mat{W}^\top \Mat{A}^\top) \\
      &= \Trace(\Mat{A} (\Mat{T}_a^{1/2} \Mat{Q} \Mat{T}_a^{-1/2}) \Mat{T}_a
      (\Mat{T}_a^{-1/2} \Mat{Q}^\top \Mat{T}_a^{1/2}) \Mat{A}^\top) \\
      &= \Trace(\Mat{A} \Mat{T}_a^{1/2} \Mat{Q} (\Mat{T}_a^{-1/2} \Mat{T}_a
      \Mat{T}_a^{-1/2}) \Mat{Q}^\top \Mat{T}_a^{1/2} \Mat{A}^\top) \\
      &= \Trace(\Mat{A} \Mat{T}_a^{1/2} (\Mat{Q} \Mat{Q}^\top)
      \Mat{T}_a^{1/2} \Mat{A}^\top) \\
      &= \Trace(\Mat{A} (\Mat{T}_a^{1/2} \Mat{T}_a^{1/2}) \Mat{A}^\top) \\
      &= \Trace(\Mat{A} \Mat{T}_a \Mat{A}^\top) \\
      &= f_1(\Mat{A})
    \end{split}
  \end{equation*}
  We can similarly prove that $f_3(\Mat{B} \Mat{W}^{-\top}) = f_3(\Mat{B})$ with
  $\Mat{W} = \Mat{T}_b^{-1/2} \Mat{Q} \Mat{T}_b^{1/2}$. Thus $f_1$ and $f_3$
  satisfy the desired invariance property.
\end{proof}

\section{Proof of Theorem~\ref{thm:symmetry-breaking}}
\label{app:symmetry-breaking}

We will need of the following lemma:
\begin{lemma}
  \label{lem:uniqueness}
  Let $\Mat{P} \in \RR^{m \times n}$ with $m \geq n$. Then the matrix equation
  \begin{equation*}
    \Mat{P} \Mat{W} = \Mat{P}, \quad \Mat{W} \in \RR^{n \times n} \text{
    orthogonal}
  \end{equation*}
  has the unique solution $\Mat{W} = \Mat{I}$ if and only if $\Mat{P}$ has full
  column rank.
\end{lemma}

\begin{proof}
  Multiplying both sides of the matrix equation by $\Mat{P}^\top$ we obtain
  $\Mat{P}^\top \Mat{P} \Mat{W} = \Mat{P}^\top \Mat{P}$. If $\Mat{P}$ has full
  column rank, then $\Mat{P}^\top \Mat{P}$ is invertible, and it follows that
  $\Mat{W} = \Mat{I}$ is the unique solution.

  Conversely, suppose $\Mat{P}$ is not full rank. Then there exists a non-zero
  $\Vec{x} \in \RR^n$ with unit norm such that $\Mat{P} \Vec{x} = \Vec{0}$. Let
  $\Mat{W} = \Mat{I} - 2 \Vec{x} \Vec{x}^\top$, then clearly
  \begin{equation*}
    \Mat{P} \Mat{W} = \Mat{P} (\Mat{I} - 2 \Vec{x} \Vec{x}^\top) = \Mat{P} - 2
    (\Mat{P} \Vec{x}) \Vec{x}^\top = \Mat{P}
  \end{equation*}
  and
  \begin{equation*}
    \Mat{W}^\top \Mat{W} = (\Mat{I} - 2 \Vec{x} \Vec{x}^\top)^\top (\Mat{I} - 2
    \Vec{x} \Vec{x}^\top) = \Mat{I} - 2 \Vec{x} \Vec{x}^\top - 2 \Vec{x}
    \Vec{x}^\top + 4 \Vec{x} (\Vec{x}^\top \Vec{x}) \Vec{x}^\top = \Mat{I}
  \end{equation*}
  since $\Vec{x}^\top \Vec{x} = 1$. Thus we have constructed a second solution
  to the matrix equation.
\end{proof}

We now prove the main result:

\noindent
{\bf Theorem~\ref{thm:symmetry-breaking}}
{\it
  Let $\Mat{T}_a$, $\Mat{T}_b$ and $\{\Lambda_1, \ldots, \Lambda_q\}$ be as
  defined in the statement of Theorem~\ref{thm:posterior-symmetry}. Define the
  prior mean matrices
  \begin{equation*}
    \Mat{M}_a =
    \begin{bmatrix}
      \Vec{\mu}_{a, 1} & \cdots & \Vec{\mu}_{a, r}
    \end{bmatrix}
    \quad \text{and} \quad
    \Mat{M}_b =
    \begin{bmatrix}
      \Vec{\mu}_{b, 1} & \cdots & \Vec{\mu}_{b, r}
    \end{bmatrix}
  \end{equation*}
  Then the posterior $p(\Mat{A}, \Mat{B} \mid \Vec{y})$ is not invariant under
  the $(\Mat{A}, \Mat{B}) \mapsto (\Mat{A} \Mat{W}, \Mat{B} \Mat{W}^{-\top})$
  transformation for any non-identity invertible $r \times r$ matrix $\Mat{W}$
  if and only if the matrices
  \begin{equation*}
    \Mat{P}_\ell =
    \begin{bmatrix}
      \Mat{M}_a[:, \Lambda_\ell] \Mat{T}_a[\Lambda_\ell, \Lambda_\ell]^{1/2} \\
      \Mat{M}_b[:, \Lambda_\ell] \Mat{T}_b[\Lambda_\ell, \Lambda_\ell]^{1/2}
    \end{bmatrix}
  \end{equation*}
  have full column rank for all $1 \leq \ell \leq q$.
}

\begin{proof}
  In this proof, we attempt to reduce the set of all possible $r \times r$
  invertible matrices $\Mat{W}$, for which invertible invariance
  \begin{equation*}
    p(\Mat{A}, \Mat{B} \mid \Vec{y}) = p(\Mat{A} \Mat{W}, \Mat{B}
    \Mat{W}^{-\top} \mid \Vec{y}) \quad \text{for all} \quad \Mat{A} \in \RR^{m
    \times r}, \Mat{B} \in \RR^{n \times r}
  \end{equation*}
  holds, to the singleton $\{\Mat{I}_r\}$. We will demonstrate that this
  reduction is possible if and only if the $\Mat{P}_\ell$ matrices (as defined
  in the theorem statement) have full column rank. We achieve this as follows:
  \begin{itemize}
    \item
      We have already shown in Proposition~\ref{prop:individual-invariance} that
      invertible invariance of the posterior holds if and only if the $f_1$,
      $f_2$, $f_3$ and $f_4$ terms (as defined in the aforementioned
      proposition) are individually invariant under the $\Mat{A} \mapsto \Mat{A}
      \Mat{W}$ and $\Mat{B} \mapsto \Mat{B} \Mat{W}^{-\top}$ transformations.
    \item
      Theorem~\ref{thm:posterior-symmetry} established that in order for the
      $f_1$ and $f_3$ terms to invariant under the transformation above,
      $\Mat{W}$ must have the structure
      \begin{equation*}
        \Mat{W} = \Mat{T}_a^{1/2} \Mat{Q} \Mat{T}_a^{-1/2} = \Mat{T}_b^{-1/2}
        \Mat{Q} \Mat{T}_b^{1/2}
      \end{equation*}
      where $\Mat{Q}$ is block-diagonal w.r.t.\ partition $\{\Lambda_1, \ldots,
      \Lambda_q\}$ as defined in the statement of the aforementioned theorem,
      and the nonzero diagonal blocks $\Mat{Q}[\Lambda_\ell, \Lambda_\ell]$ are
      orthogonal for all $1 \leq \ell \leq q$.
    \item
      In Step 1 below, we consider the terms $f_2$ and $f_4$, and derive simpler
      and equivalent conditions on matrix $\Mat{W}$ (more specifically, the
      matrix $\Mat{Q}$) to ensure invariance under the $\Mat{A} \mapsto \Mat{A}
      \Mat{W}$ and $\Mat{B} \mapsto \Mat{B} \Mat{W}^{-\top}$ transformations.
      These conditions are formulated in terms of the prior means $\Vec{\mu}_{a,
      k}$, $\Vec{\mu}_{b, k}$ and precisions $\tau_{a, k}$, $\tau_{b, k}$ for
      invariance.
    \item
      In Step 2, we further analyze these simpler conditions and frame them as
      matrix equations on diagonal blocks of $\Mat{Q}$.
    \item
      Finally, in Step 3, we will use Lemma~\ref{lem:uniqueness} to demonstrate
      $\Mat{W} = \Mat{I}$ is the only solution of this matrix system if and only
      if the matrices $\Mat{P}_\ell$ are full rank for $1 \leq \ell \leq q$.
  \end{itemize}

  \noindent\textbf{Step 1.} Let us explicitly write out the invariance of $f_2$:
  we have $f_2(\Mat{A} \Mat{W}) = f_2(\Mat{A})$, i.e.
  \begin{align}
    \sum_{k' = 1}^r \tau_{a,k'} \Vec{\mu}_{a,k'}^\top (\Mat{A} \Mat{W})_{k'} &=
    \sum_{k' = 1}^r \tau_{a,k'} \Vec{\mu}_{a,k'}^\top \Vec{a}_{k'} \nonumber \\
    \label{eq:f2-invariance}
    \implies \sum_{k' = 1}^r \tau_{a,k'} \Vec{\mu}_{a,k'}^\top \Mat{A} \Mat{W}
    \Vec{e}_{k'} &= \sum_{k' = 1}^r \tau_{a,k'} \Vec{\mu}_{a,k'}^\top
    \Vec{a}_{k'}
  \end{align}
  for all $\Mat{A}$. Note that we can write
  \begin{equation*}
    \Mat{A} = \sum_{i' = 1}^r \Vec{a}_{i'} \Vec{e}_{i'}^\top
  \end{equation*}
  Substituting this expression on the left hand size of
  \eqref{eq:f2-invariance}, we obtain
  \begin{equation*}
    \begin{split}
      \sum_{k' = 1}^r \tau_{a,k'} \Vec{\mu}_{a,k'}^\top \Mat{A} \Mat{W}
      \Vec{e}_{k'}
      &= \sum_{k' = 1}^r \tau_{a,k'} \Vec{\mu}_{a,k'}^\top \left(\sum_{i' = 1}^r
      \Vec{a}_{i'} \Vec{e}_{i'}^\top\right) \Mat{W} \Vec{e}_{k'} \\
      &= \sum_{k' = 1}^r \tau_{a,k'} \Vec{\mu}_{a,k'}^\top \sum_{i' = 1}^r
      \Vec{a}_{i'} \Vec{e}_{i'}^\top (\Mat{T}_a^{1/2} \Mat{Q} \Mat{T}_a^{-1/2})
      \Vec{e}_{k'} \\
      &= \sum_{k' = 1}^r \tau_{a,k'} \Vec{\mu}_{a,k'}^\top \sum_{i' = 1}^r
      \Vec{a}_{i'} (\Vec{e}_{i'}^\top \Mat{T}_a^{1/2}) \Mat{Q} (\Mat{T}_a^{-1/2}
      \Vec{e}_{k'}) \\
      &= \sum_{k' = 1}^r \tau_{a,k'} \Vec{\mu}_{a,k'}^\top \sum_{i' = 1}^r
      \Vec{a}_{i'} (\tau_{a, i'}^{1/2} \Vec{e}_{i'}^\top) \Mat{Q} (\tau_{a,
      k'}^{-1/2} \Vec{e}_{k'}) \\
      &= \sum_{k' = 1}^r \tau_{a,k'}^{1/2} \Vec{\mu}_{a,k'}^\top \sum_{i' = 1}^r
      \tau_{a, i'}^{1/2} \Vec{a}_{i'} (\Vec{e}_{i'}^\top \Mat{Q} \Vec{e}_{k'})
      \\
      &= \sum_{k' = 1}^r \tau_{a,k'}^{1/2} \Vec{\mu}_{a,k'}^\top \sum_{i' = 1}^r
      \tau_{a, i'}^{1/2} \Vec{a}_{i'} q_{i', k'} \\
      &= \sum_{k' = 1}^r \Vec{\tilde{\mu}}_{a, k'}^\top \sum_{i' = 1}^r
      \Vec{\tilde{a}}_{i'} q_{i', k'}
    \end{split}
  \end{equation*}
  where we denote
  \begin{equation*}
    \Vec{\tilde{\mu}}_{a, k'} = \tau_{a, k'}^{1/2} \Vec{\mu}_{a, k'}, \quad
    \Vec{\tilde{a}}_{k'} = \tau_{a, k'}^{1/2} \Vec{a}_{k'}, \quad k' \in \{1,
    \ldots, r\}
  \end{equation*}
  The right hand side of \eqref{eq:f2-invariance} can be rewritten using
  this notation as
  \begin{equation*}
    \sum_{k' = 1}^r \tau_{a,k'} \Vec{\mu}_{a,k'}^\top \Vec{a}_{k'} = \sum_{k' =
    1}^r \Vec{\tilde{\mu}}_{a,k'}^\top \Vec{\tilde{a}}_{k'}
  \end{equation*}
  These two computations simplifies \eqref{eq:f2-invariance} to
  \begin{equation*}
    \sum_{k' = 1}^r \Vec{\tilde{\mu}}_{a, k'}^\top \sum_{i' = 1}^r
    \Vec{\tilde{a}}_{i'} q_{i', k'} = \sum_{k' = 1}^r \Vec{\tilde{\mu}}_{a,
    k'}^\top \Vec{\tilde{a}}_{k'}
  \end{equation*}
  Switching the order of summation on the left side, changing the summation
  index on the right side, and using $\Vec{\tilde{a}}_{i'}^\top
  \Vec{\tilde{\mu}}_{a, k'} = \Vec{\tilde{\mu}}_{a, k'}^\top
  \Vec{\tilde{a}}_{i'}$, we obtain
  \begin{equation}
    \label{eq:f2-invariance-simple}
    \sum_{i' = 1}^r \Vec{\tilde{a}}_{i'}^\top \sum_{k' = 1}^r
    \Vec{\tilde{\mu}}_{a,k'} q_{i', k'} = \sum_{i' = 1}^r
    \Vec{\tilde{a}}_{i'}^\top \Vec{\tilde{\mu}}_{a, i'}
  \end{equation}
  It has to hold for any arbitrary $\Vec{\tilde{a}}_{i'} \in \RR^{m}$ for $i'
  \in \{1, \ldots, r\}$ (since the columns $\Vec{a}_{i'}$ of matrix $\Mat{A}$
  are arbitrary and $\tau_{a, i'}$ are positive reals). Let us fix $1 \leq i
  \leq r$ and assume all but the $i$-th of these vectors $\Vec{\tilde{a}}_{i'}$
  are zeros. Then \eqref{eq:f2-invariance-simple} reduces to
  \begin{equation*}
    \Vec{\tilde{a}}_i^\top \sum_{k' = 1}^r \Vec{\tilde{\mu}}_{a,k'} q_{i, k'} =
    \Vec{\tilde{a}}_i^\top \Vec{\tilde{\mu}}_{a, i}
  \end{equation*}
  Since this holds for arbitrary $\Vec{\tilde{a}}_i \in \RR^m$, we conclude
  \begin{equation}
    \label{eq:f2-invariance-individual}
    \sum_{k' = 1}^r \Vec{\tilde{\mu}}_{a,k'} q_{i, k'} = \Vec{\tilde{\mu}}_{a,
    i}
  \end{equation}
  Conversely, if \eqref{eq:f2-invariance-individual} holds for all $1
  \leq i \leq r$, then \eqref{eq:f2-invariance-simple} is trivially
  satisfied.

  Let us pause and review our progress. Under the $\Mat{A} \mapsto \Mat{A}
  \Mat{W}$ transformation, where $\Mat{W}$ has the form defined in
  \eqref{eq:posterior-symmetry-cond-1} and
  \eqref{eq:posterior-symmetry-cond-2} (required for invariance of the
  $f_1$ and $f_3$ terms, c.f.\ Theorem~\ref{thm:posterior-symmetry}), the term
  $f_2$ is invariant if and only if identity \eqref{eq:f2-invariance}
  holds if and only if identity \eqref{eq:f2-invariance-simple} holds if
  and only if equation \eqref{eq:f2-invariance-individual} is true.

  Note that $\Mat{W}^{-\top} = \Mat{T}_b^{1/2} \Mat{Q}^{-\top} \Mat{T}_b^{-1/2}
  = \Mat{T}_b^{1/2} \Mat{Q} \Mat{T}_b^{-1/2}$ where the last equality follows
  from orthogonality of $\Mat{Q}$. We can now repeat the same process as above,
  and establish that $f_4$ is invariant under the $\Mat{B} \mapsto \Mat{B}
  \Mat{W}^{-\top}$ transformation if and only if
  \begin{equation}
    \label{eq:f4-invariance-individual}
    \sum_{k' = 1}^r \Vec{\tilde{\mu}}_{b,k'} q_{i, k'} = \Vec{\tilde{\mu}}_{b,
    i}
  \end{equation}
  holds.

  We combine these two arguments, and conclude $f_2$ and $f_4$ are invariant
  (after assuming the conditions \eqref{eq:posterior-symmetry-cond-1} and
  \eqref{eq:posterior-symmetry-cond-2} equivalent to invariances of $f_1$
  and $f_3$) if and only if \eqref{eq:f2-invariance-individual} and
  \eqref{eq:f4-invariance-individual} holds.

  \noindent\textbf{Step 2.} We now frame
  \eqref{eq:f2-invariance-individual} and
  \eqref{eq:f4-invariance-individual} as matrix equations for the
  diagonal blocks of the $\Mat{Q}$ matrix. In
  \eqref{eq:f2-invariance-individual}, let us assume $i \in \Lambda_\ell$
  for some $\ell \in \{1, \ldots, q\}$. Then, since $\Mat{Q}$ is block-diagonal
  w.r.t.\ partitions $\{\Lambda_1, \ldots, \Lambda_\ell\}$, we have
  \begin{equation*}
    q_{i, k'} = 0 \quad \text{for all} \quad k' \not\in \Lambda_\ell
  \end{equation*}
  and \eqref{eq:f2-invariance-individual} further reduces to
  \begin{equation*}
    \sum_{k' \in \Lambda_\ell} \Vec{\tilde{\mu}}_{a, k'} q_{i, k'} =
    \Vec{\tilde{\mu}}_{a, i}
    \implies
    \sum_{k' \in \Lambda_\ell} \tau_{a, k'}^{1/2} \Vec{\mu}_{a, k'} q_{i, k'} =
    \tau_{a, i}^{1/2} \Vec{\mu}_{a, i}
  \end{equation*}
  This is a linear system with unknown $\Mat{Q}[i, \Lambda_\ell]$; in matrix
  form, we can write it as
  \begin{equation}
    \label{eq:f2-invariance-matrix}
    \Mat{M}_a[:, \Lambda_\ell] \Mat{T}_a[\Lambda_\ell, \Lambda_\ell]^{1/2}
    \Mat{Q}[i, \Lambda_\ell]^\top = \tau_{a, i}^{1/2} \Vec{\mu}_{a, i}
  \end{equation}
  We can similarly pose \eqref{eq:f4-invariance-individual} as a matrix
  equation
  \begin{equation}
    \label{eq:f4-invariance-matrix}
    \Mat{M}_b[:, \Lambda_\ell] \Mat{T}_b[\Lambda_\ell, \Lambda_\ell]^{1/2}
    \Mat{Q}[i, \Lambda_\ell]^\top = \tau_{b, i}^{1/2} \Vec{\mu}_{b, i}
  \end{equation}
  Combining \eqref{eq:f2-invariance-matrix} and
  \eqref{eq:f4-invariance-matrix} for all $i \in \Lambda_\ell$, we obtain
  the system
  \begin{equation*}
    \begin{bmatrix}
      \Mat{M}_a[:, \Lambda_\ell] \Mat{T}_a[\Lambda_\ell, \Lambda_\ell]^{1/2} \\
      \Mat{M}_b[:, \Lambda_\ell] \Mat{T}_b[\Lambda_\ell, \Lambda_\ell]^{1/2}
    \end{bmatrix}
    \Mat{Q}[\Lambda_\ell, \Lambda_\ell]^\top
    =
    \begin{bmatrix}
      \Mat{M}_a[:, \Lambda_\ell] \Mat{T}_a[\Lambda_\ell, \Lambda_\ell]^{1/2} \\
      \Mat{M}_b[:, \Lambda_\ell] \Mat{T}_b[\Lambda_\ell, \Lambda_\ell]^{1/2}
  \end{bmatrix}
  \end{equation*}
  Denoting the matrix on the right hand side as $\Mat{P}_\ell$, we obtain
  \begin{equation}
    \label{eq:invariance-cond}
    \Mat{P}_\ell \Mat{Q}[\Lambda_\ell, \Lambda_\ell]^\top = \Mat{P}_\ell
  \end{equation}

  In summary, given the block-diagonal structure of $\Mat{Q}$ from invariance of
  $f_1$ and $f_3$ terms, we have derived matrix equation
  \eqref{eq:invariance-cond} which is equivalent to
  \eqref{eq:f2-invariance-individual} and
  \eqref{eq:f4-invariance-individual}.  These later two conditions are
  both necessary and sufficient for invariance of $f_2$ and $f_4$ to hold.

  \noindent\textbf{Step 3.} By Lemma~\ref{lem:uniqueness}, the solution
  $\Mat{Q}[\Lambda_\ell, \Lambda_\ell] = \Mat{I}_{r_\ell}$ of
  \eqref{eq:invariance-cond} among orthogonal $r_\ell \times r_\ell$
  matrices is unique if and only if the matrix $\Mat{P}_\ell$ has full column
  rank.  Collecting this result for all $\ell \in \{1, \ldots, q\}$, we conclude
  that $\Mat{Q} = \Mat{I}$ is the unique matrix generating the invertible
  invariance matrix $\Mat{W}$ if and only if the matrices $\Mat{P}_\ell$ are
  full rank.

  Finally note that
  \begin{equation*}
    \Mat{Q} = \Mat{I} \implies \Mat{W} = \Mat{T}_a^{1/2} \Mat{Q}
    \Mat{T}_a^{-1/2} = \Mat{T}_a^{1/2} \Mat{I} \Mat{T}_a^{-1/2} = \Mat{I}
  \end{equation*}
  and
  \begin{equation*}
    \Mat{W} = \Mat{I} \implies \Mat{Q} = \Mat{T}_a^{-1/2} \Mat{W}
    \Mat{T}_a^{1/2} = \Mat{T}_a^{-1/2} \Mat{I} \Mat{T}_a^{1/2} = \Mat{I}
  \end{equation*}
  Thus $\Mat{Q} = \Mat{I}$ if and only if $\Mat{W} = \Mat{I}$ and we conclude
  our proof.
\end{proof}

\vskip 0.2in
\bibliography{references}

\end{document}